\documentclass[sigconf]{aamas}

\usepackage{balance} 

\doi{UUEW5708}

\makeatletter
\gdef\@copyrightpermission{
  \begin{minipage}{0.2\columnwidth}
   \href{https://creativecommons.org/licenses/by/4.0/}{\includegraphics[width=0.90\textwidth]{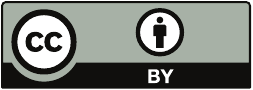}}
  \end{minipage}\hfill
  \begin{minipage}{0.8\columnwidth}
   \href{https://creativecommons.org/licenses/by/4.0/}{This work is licensed under a Creative Commons Attribution International 4.0 License.}
  \end{minipage}
  \vspace{5pt}
}
\makeatother

\setcopyright{ifaamas}
\acmConference[AAMAS '26]{Proc.\@ of the 25th International Conference
on Autonomous Agents and Multiagent Systems (AAMAS 2026)}{May 25 -- 29, 2026}
{Paphos, Cyprus}{C.~Amato, L.~Dennis, V.~Mascardi, J.~Thangarajah (eds.)}
\copyrightyear{2026}
\acmYear{2026}
\acmDOI{}
\acmPrice{}
\acmISBN{}

\acmSubmissionID{222}


\usepackage[utf8]{inputenc} 
\usepackage[T1]{fontenc}    
\usepackage{hyperref}       
\usepackage{url}            
\usepackage{booktabs}       
\usepackage{nicefrac}       
\usepackage{xcolor}         

\usepackage{amsfonts}       
\usepackage{amsmath}
\usepackage{mathtools}
\usepackage{amsthm}
\usepackage{bbm}
\usepackage{algorithm}
\usepackage{algpseudocode}
\usepackage{subcaption}

\usepackage[capitalize,noabbrev]{cleveref}

\theoremstyle{plain}
\newtheorem{theorem}{Theorem}[section]
\newtheorem{proposition}[theorem]{Proposition}

\theoremstyle{definition}
\newtheorem{definition}[theorem]{Definition}

\theoremstyle{remark}

\usepackage{enumitem}
\usepackage{mathrsfs}
\usepackage[on]{editing}
\usepackage{xcolor}
\usepackage{wrapfig}
\usepackage{multirow}
\usepackage{booktabs} 
\usepackage{graphicx} 
\usepackage{adjustbox} 
\usepackage{caption}   
\usepackage{makecell} 

\newcommand{\interv}[1]{\operatorname{do}({#1})}

\Crefname{equation}{Eq.}{Eqs.}
\Crefname{align}{Eq.}{Eqs.}
\Crefname{figure}{Fig.}{Figs.}
\Crefname{tabular}{Tab.}{Tabs.}
\Crefname{theorem}{Thm.}{Thms.}
\Crefname{lemma}{Lem.}{Lems.}
\Crefname{proposition}{Prop.}{Props.}
\Crefname{definition}{Def.}{Defs.}
\Crefname{algorithm}{Algo.}{Algs.}
\Crefname{corollary}{Corol.}{Corol.}
\Crefname{section}{Sec.}{Sec.}
\Crefname{appendix}{App.}{Apps.}
\Crefname{prop}{Prop.}{Props.}
\Crefname{task}{Task.}{Tasks.}
\Crefname{setting}{Setting.}{Settings.}
\Crefname{example}{Example}{Expamples}

\newcommand{\Parens}[1]{\left(#1\right)}

\newcommand{\I}{\boldsymbol{1}}

\newcommand{\doo}{\text{do}}

\def\*#1{\boldsymbol{#1}}
\def\1#1{\mathcal{#1}}
\def\2#1{\mathscr{#1}}
\def\3#1{\mathbb{#1}}

\usepackage{pgfplots}
\DeclareUnicodeCharacter{2212}{−}
\usepgfplotslibrary{groupplots,dateplot}
\usetikzlibrary{patterns,shapes.arrows}
\pgfplotsset{compat=newest}

\definecolor{betterred}{RGB}{228,26,28}
\definecolor{betterblue}{RGB}{55,126,184}
\definecolor{betteryellow}{RGB}{255, 191, 0}
\definecolor{betteryellow2}{RGB}{248, 222, 34}
\definecolor{bettergreen}{RGB}{77,175,74}
\definecolor{betterpurple}{RGB}{152,78,163}
\definecolor{LightCyan}{rgb}{0.88,1,1}

\usetikzlibrary{positioning,arrows,shapes,shapes.arrows,arrows.meta,trees,shapes.misc,shapes.geometric,decorations.pathreplacing,automata}
\tikzset{%
  >={Latex[width=1.5mm,length=2mm]},
  vertex/.style={draw,circle,inner sep=0mm,semithick,minimum width=4mm},
  point/.style = {circle, draw, inner sep=0.04cm,fill,node contents={}},
  uvertex/.style={outer sep=0},
  bidir/.style={<->,dashed,betteryellow,line width=0.45mm},
  dir/.style={->, line width=0.25mm},
  regime/.style={shape=rectangle,fill=black,inner sep=0pt,minimum size=3pt,draw},
  action/.style={shape=circle,fill=black,inner sep=0pt,minimum size=8pt,draw},
  node distance=1cm,
  font=\scriptsize\sffamily
}

\pgfdeclarelayer{back}
\pgfsetlayers{back,main}

\makeatletter
\pgfkeys{%
  /tikz/on layer/.code={
    \def\tikz@path@do@at@end{\endpgfonlayer\endgroup\tikz@path@do@at@end}%
    \pgfonlayer{#1}\begingroup%
  }%
}
\makeatother

\makeatletter
\newcommand{\xdashleftrightarrow}[2][]{\ext@arrow 3359\leftrightarrowfill@@{#1}{#2}}
\def\rightarrowfill@@{\arrowfill@@\relax\relbar\rightarrow}
\def\leftarrowfill@@{\arrowfill@@\leftarrow\relbar\relax}
\def\leftrightarrowfill@@{\arrowfill@@\leftarrow\relbar\rightarrow}
\def\arrowfill@@#1#2#3#4{%
  $\m@th\thickmuskip0mu\medmuskip\thickmuskip\thinmuskip\thickmuskip
   \relax#4#1
   \xleaders\hbox{$#4#2$}\hfill
   #3$%
}
\makeatother

\usepackage{titletoc}
\newcommand\DoToC{%
  \setcounter{tocdepth}{4} 
  \startcontents
  \printcontents{}{1}{\textbf{Contents}\vskip3pt\hrule\vskip5pt}
  \vskip3pt\hrule\vskip5pt
}

\title{Confounding Robust Continuous Control \\ via Automatic Reward Shaping}

\author{Mateo Juliani}
\authornote{Both authors contributed equally to this research.}
\orcid{0009-0003-2502-3815} 
\affiliation{
  \institution{Columbia University}
  \city{New York, NY}
  \country{United States}}
\email{msj2164@columbia.edu}

\author{Mingxuan Li}
\authornotemark[1]
\orcid{0000-0002-2246-1938}
\affiliation{
  \institution{Columbia University}
  \city{New York, NY}
  \country{United States}}
\email{ml@cs.columbia.edu}

\author{Elias Bareinboim}
\orcid{0009-0006-9439-921X}
\affiliation{
  \institution{Columbia University}
  \city{New York, NY}
  \country{United States}}
\email{eb@cs.columbia.edu}

\begin{abstract}
Reward shaping has been applied widely to accelerate Reinforcement Learning (RL) agents' training. However, a principled way of designing effective reward shaping functions, especially for complex continuous control problems, remains largely under-explained. In this work, we propose to automatically learn a reward shaping function for continuous control problems from offline datasets, potentially contaminated by unobserved confounding variables. Specifically, our method builds upon the recently proposed causal Bellman equation to learn a tight upper bound on the optimal state values, which is then used as the potentials in the Potential-Based Reward Shaping (PBRS) framework. 
Our proposed reward shaping algorithm is tested with Soft-Actor-Critic (SAC) on multiple commonly used continuous control benchmarks and exhibits strong performance guarantees under unobserved confounders. More broadly, our work marks a solid first step towards confounding robust continuous control from a causal perspective. Code for training our reward shaping functions can be found at \textcolor{betterblue}{\url{https://github.com/mateojuliani/confounding_robust_cont_control}}.
\end{abstract}


\begin{document}
\pagestyle{fancy}
\fancyhead{}
\maketitle

\section{Introduction}
Reinforcement learning (RL) has demonstrated impressive success in continuous control domains such as robotic manipulation, locomotion, and autonomous systems \citep{schulmanProximalPolicyOptimization2017,schulmanHighDimensionalContinuousControl2016a,mnihHumanlevelControlDeep2015a}. Despite this progress, learning effective policies in high-dimensional, complex environments remains challenging due to sample inefficiency and high sensitivity to reward design. When the original task reward is not efficient to learn from, reward shaping can significantly accelerate learning by injecting informative signals that guide exploration and policy improvement. However, designing effective shaping functions remains a persistent challenge, often requiring substantial domain expertise and manual effort to ensure they are helpful without hurting the performance \citep{DBLP:conf/icml/RandlovA98,pbrs,shapingsurvey2024,DBLP:conf/iclr/PanBS22}.

Potential-Based Reward Shaping (PBRS) \citep{pbrs} offers a principled framework for injecting additional reward signals while preserving the original task’s optimal policy. However, the effectiveness of PBRS critically depends on the quality of the potential function used. Recent work has explored learning state potentials automatically from offline data~\citep{shapingdemonstration,DBLP:conf/corl/MezghaniSBLA22, DBLP:conf/cikm/ZhangQL024}, but such approaches typically assume no unobserved confounders (NUC) \citep{murphy2003optimal,murphy2005generalization}, i.e. access to fully observed (unconfounded) trajectories. Such an assumption can easily break down in many real-world settings. 
Unobserved confounders can arise from human demonstrations, legacy systems, or sensor capability differences in robotic platforms. When the NUC assumption is violated, the effects of candidate policies become generally unidentifiable. That is, the given model assumptions are insufficient to uniquely recover the value function from offline data, regardless of the sample size \citep{pearl2009causality,zhang2019near}. As a result, standard RL methods relying implicitly on NUC can suffer from degraded learning performance in such settings. More recently, \citet{li2025automatic} propose to use partial identification approach to learn confounding robust shaping functions. But their proposed method is limited to discrete and lower dimensional settings. A practical solution for continuous and higher dimensional environments is yet to be discussed. 

In this work, we tackle the problem of automatic reward shaping in continuous control settings where offline data may be confounded. Specifically, we utilize offline trajectories collected from unknown, potentially biased behavioral policies to estimate causal upper bounds on the optimal interventional state values. These bounds are then employed as potential functions within the PBRS framework to construct shaping rewards. By incorporating these shaped rewards, we enable model-free RL agents to perform more informed exploration and policy learning, even in the presence of unobserved confounders.
We empirically evaluate our framework in confounded MuJoCo and Adroit environments with partial observability. Experiments on a suite of confounded continuous environments show that our method consistently outperforms unshaped and causally unaware shaping baselines (CQL-Shaping, \citep{CQL2020Kumar}, T-Rex-Shaping \cite{brown2019extrapolating}). These results highlight the robustness and practical effectiveness of our approach in real-world settings.

Our main contributions are as follows:
\begin{enumerate}[label=\textbullet, leftmargin=20pt, topsep=0pt, parsep=0pt, itemsep=1pt]
    \item We derive a Causal Bellman Equation for the stationary infinite-horizon Confounded Markov Decision Process (CMDP), which converges to a tight upper bound of the optimal state values;
    \item We design a neural learning algorithm that approximates the Causal Bellman Equation 
    in high-dimensional continuous state-action CMDPs;
    \item We use the learned state value bound as the reward shaping function and empirically demonstrate the superior performance boost when applying to Soft-Actor-Critic \citep{haarnoja2018soft} in various challenging confounded continuous control environments. 
\end{enumerate}

\section{Background}
\paragraph{Confounding Robust Decision-making}
We will focus on the sequential decision-making setting in an infinite horizon stationary Markov Decision Process (MDP, \citet{putermanMarkovDecisionProcesses1994book}) where the agent intervenes on a sequence of actions $X_1, \dots$ to optimize the cumulative return over reward signals $Y_1, \dots$ given state observations $S_1, S_2, \dots$ at each corresponding time step. 
Standard MDP formalism focuses on the perspective of the learners who could actively intervene in the environment. Consequently, the data collected from randomized experiments is free from the contamination of unobserved confounding bias and is generally assumed away in the model. However, when considering offline data collected by passive observation \citep{levine2020offline,lupessimism,CQL2020Kumar} where the learner may not necessarily have deliberate control over the behavioral policy generating the data, or when the state attributes are partially observed \citep{DBLP:conf/nips/JaakkolaSJ94,kaelbling1998planning,tennenholtz2020off}, unobserved confounding arises. Consequently, this could lead to biased estimation and safety/alignment issue in various reinforcement learning tasks, including off-policy learning \citep{kallus2018confounding,lupessimism,zhang2024eligibility,li2025automatic,li2025cdqn}, curriculum learning \citep{li2024causally}, and imitation learning \citep{zhang2020causal,sequentialcausal}.

\paragraph{Continuous Control with Deep Reinforcement Learning} In continuous action space, the sample efficiency problem is exacerbated, rendering commonly used on-policy learning solutions unfavorable, such as TRPO \citep{10.5555/3045118.3045319}, PPO \citep{schulmanProximalPolicyOptimization2017} or A3C \citep{DBLP:conf/icml/MnihBMGLHSK16}. At each time step, on-policy algorithms collect new trajectories from the environment only for updating the agents by a single gradient step. As task complexity grows, this procedure becomes increasingly expensive. Off-policy algorithms, on the other hand, reuse past experiences. The direct application of this idea is DQN and its variants \citep{mnihHumanlevelControlDeep2015a}. For continuous policy learning, actor-critic based method is preferred for its stability and easy-to-tune hyper-parameters \citep{schulmanHighDimensionalContinuousControl2016a,haarnoja2018soft}. In this work, we use Soft-Actor-Critic (SAC), a maximum entropy reinforcement learning framework which improves upon the traditional maximum reward framework with substantially better exploration and robustness \citep{DBLP:conf/aaai/ZiebartMBD08,DBLP:conf/icml/HaarnojaTAL17}, as our base learning algorithm.

\paragraph{Potential-based reward shaping (PBRS)}
Reward shaping is a popular line of techniques for incorporating domain knowledge during policy learning. Common approaches such as Potential-Based Reward Shaping (PBRS, \citet{pbrs}) add supplemental signals to the reward function so that it would be easier to learn in future downstream tasks without affecting the optimality of the learned policy. PBRS modifies the reward function in the system by adding the discounted next state potential subtracted by current state potential. This encourages the learning agent to visit states with higher potentials while avoiding visiting states with low potentials. More importantly, optimal policies remain invariant across this shaping process, i.e., every optimal policy learned in the MDP under PBRS is guaranteed to be optimal in the original MDP, and vice versa.

\paragraph{Notations}
We will consistently use capital letters ($V$) to denote random variables, lowercase letters ($v$) for their values, and cursive $\mathcal{V}$ to denote their domains. Fix indices $i, j \in \3N$. We use bold capital letters ($\boldsymbol{V}$) to denote a set of random variables and let $|\boldsymbol{V}|$ denote its cardinality of the set $\boldsymbol{V}$. Finally, $\I_{\*Z = \*z}$ is an indicator function that returns $1$ if event $\*Z = \*z$ holds true; otherwise, it returns $0$. 
\section{The Challenge of Confounded Continuous Control}
In real-world continuous control tasks, agents may have limited sensor information due to cost or environment constraints when deployed in online operation. This leads to poor sample efficiency if the agent is to be tuned fully online. Pre-training with offline data is a common remedy. But people generally expect the behavioral policy generating the offline datasets to have the same sensor capability as the online agent \cite{levine2020offline,calql,CQL2020Kumar}. However, offline data might be collected by agents in a more controlled environment with privileged sensor capability \citep{DBLP:conf/rss/MonaciAS22,hu2024privileged,huang2025pigdreamer}. This mismatch induces unobserved confounding at the decision level: offline decisions reflect privileged information unavailable to the online learning agent. As a result, directly learning from such data without addressing confounding can lead to biased policies and poor online performance. Tackling confounded continuous control is thus critical for reliable and sample-efficient RL under realistic sensory constraints. 

In this paper, we consider an extended family of MDPs that explicitly models the presence of unobserved confounders when generating offline data. 

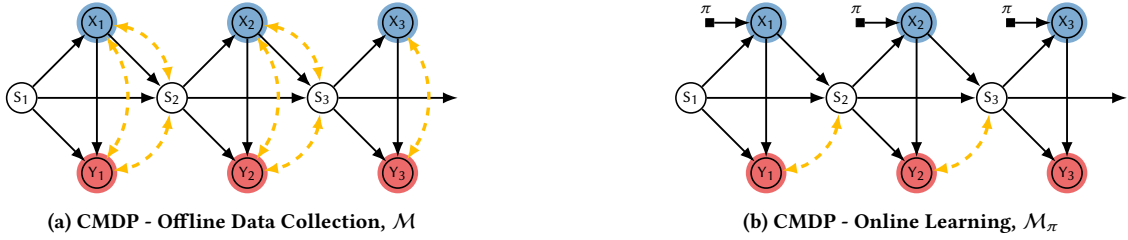
\begin{figure*}[t]
\centering
    \begin{subfigure}[b]{.5\linewidth}\centering
      \begin{tikzpicture}
        \def\outerr{3}
        \def\innerr{2.7}
        \node[vertex] at (0, 0) (S1) {S\textsubscript{1}};
        \node[vertex] at (1, 1) (X1) {X\textsubscript{1}};
        \node[vertex] at (1, -1) (Y1) {Y\textsubscript{1}};
        \node[vertex] at (2, 0) (S2) {S\textsubscript{2}};
        \node[vertex] at (3, 1) (X2) {X\textsubscript{2}};
        \node[vertex] at (3, -1) (Y2) {Y\textsubscript{2}};
        \node[vertex] at (4, 0) (S3) {S\textsubscript{3}};
        \node[vertex] at (5, 1) (X3) {X\textsubscript{3}};
        \node[vertex] at (5, -1) (Y3) {Y\textsubscript{3}};
        \draw[dir] (S1) edge [bend left=0] (Y1);
        \draw[dir] (X1) edge [bend left=0] (Y1);
        \draw[dir] (X2) edge [bend left=0] (Y2);
        \draw[dir] (S1) edge [bend left=0] (S2);
        \draw[dir] (X1) edge [bend left=0] (S2);
        \draw[dir] (S2) edge [bend left=0] (Y2);
        \draw[dir] (S2) edge [bend left=0] (S3);
        \draw[dir] (X2) edge [bend left=0] (S3);
        \draw[dir] (X3) edge [bend left=0] (Y3);
        \draw[dir] (S3) edge [bend left=0] (Y3);
        \draw[dir] (S1) edge [bend left=0] (X1);
        \draw[dir] (S2) edge [bend left=0] (X2);
        \draw[dir] (S3) edge [bend left=0] (X3);
        \draw[dir] (S3) edge [bend left=0] +(1.8,0);
        \draw[bidir] (X1) edge [bend left=36] (S2);
        \draw[bidir] (X2) edge [bend left=36] (S3);
        \draw[bidir] (S2) edge [bend left=36] (Y1);
        \draw[bidir] (S3) edge [bend left=36] (Y2);
        \draw[bidir] (X1) edge [bend left=34] (Y1);
        \draw[bidir] (X2) edge [bend left=34] (Y2);
        \draw[bidir] (X3) edge [bend left=34] (Y3);
        \begin{pgfonlayer}{back}
          \node[circle,fill=betterblue!65,draw=none,minimum size=2*\innerr mm] at (X1) {};
		  \node[circle,fill=betterblue!65,draw=none,minimum size=2*\innerr mm] at (X2) {};
		  \node[circle,fill=betterblue!65,draw=none,minimum size=2*\innerr mm] at (X3) {};
		  \node[circle,fill=betterred!65,draw=none,minimum size=2*\innerr mm] at (Y1) {};
		  \node[circle,fill=betterred!65,draw=none,minimum size=2*\innerr mm] at (Y2) {};
		  \node[circle,fill=betterred!65,draw=none,minimum size=2*\innerr mm] at (Y3) {};
	  \end{pgfonlayer}
      \end{tikzpicture}
      \caption{CMDP - Offline Data Collection, $\mathcal{M}$}
      \label{fig:cmdp offline}
    \end{subfigure}\hfill
    \begin{subfigure}[b]{.5\linewidth}\centering
      \begin{tikzpicture}            
        \def\outerr{3}
        \def\innerr{2.7}
        \node[vertex] at (0, 0) (S1) {S\textsubscript{1}};
        \node[vertex] at (1, 1) (X1) {X\textsubscript{1}};
        \node[vertex] at (1, -1) (Y1) {Y\textsubscript{1}};
        \node[vertex] at (2, 0) (S2) {S\textsubscript{2}};
        \node[vertex] at (3, 1) (X2) {X\textsubscript{2}};
        \node[vertex] at (3, -1) (Y2) {Y\textsubscript{2}};
        \node[vertex] at (4, 0) (S3) {S\textsubscript{3}};
        \node[vertex] at (5, 1) (X3) {X\textsubscript{3}};
        \node[vertex] at (5, -1) (Y3) {Y\textsubscript{3}};
        \node[regime, label={[shift={(-0.05,-0.05)}] \scriptsize $\pi$ }] (p1) at (0.25, 1) {};
        \node[regime, label={[shift={(-0.05,-0.05)}] \scriptsize $\pi$ }] (p2) at (2.25, 1) {};
        \node[regime, label={[shift={(-0.05,-0.05)}] \scriptsize $\pi$ }] (p3) at (4.25, 1) {};
        \draw[dir] (S1) edge [bend left=0] (Y1);
        \draw[dir] (X1) edge [bend left=0] (Y1);
        \draw[dir] (X2) edge [bend left=0] (Y2);
        \draw[dir] (S1) edge [bend left=0] (S2);
        \draw[dir] (X1) edge [bend left=0] (S2);
        \draw[dir] (S2) edge [bend left=0] (Y2);
        \draw[dir] (S2) edge [bend left=0] (S3);
        \draw[dir] (X2) edge [bend left=0] (S3);
        \draw[dir] (X3) edge [bend left=0] (Y3);
        \draw[dir] (S3) edge [bend left=0] (Y3);
        \draw[dir] (S1) edge [bend left=0] (X1);
        \draw[dir] (S2) edge [bend left=0] (X2);
        \draw[dir] (S3) edge [bend left=0] (X3);
        \draw[dir] (S3) edge [bend left=0] +(1.8,0);
        \draw[bidir] (S2) edge [bend left=36] (Y1);
        \draw[bidir] (S3) edge [bend left=36] (Y2);
        \draw[dir] (p1) to (X1);
	    \draw[dir] (p2) to (X2);
	    \draw[dir] (p3) to (X3);
        \begin{pgfonlayer}{back}
          \node[circle,fill=betterblue!65,draw=none,minimum size=2*\innerr mm] at (X1) {};
		  \node[circle,fill=betterblue!65,draw=none,minimum size=2*\innerr mm] at (X2) {};
		  \node[circle,fill=betterblue!65,draw=none,minimum size=2*\innerr mm] at (X3) {};
		  \node[circle,fill=betterred!65,draw=none,minimum size=2*\innerr mm] at (Y1) {};
		  \node[circle,fill=betterred!65,draw=none,minimum size=2*\innerr mm] at (Y2) {};
		  \node[circle,fill=betterred!65,draw=none,minimum size=2*\innerr mm] at (Y3) {};
	  \end{pgfonlayer}
      \end{tikzpicture}
      \caption{CMDP - Online Learning, $\mathcal{M}_\pi$}
      \label{fig:cmdp online}
    \end{subfigure}\hfill
\label{fig:cmdp}
\vspace{-.1in}
\caption{(a) CMDP causal diagram of the offline data generating process; (b) CMDP causal diagram under policy $\interv{\pi}$ during the online learning process. Compared with the standard MDP, the highlighted bi-directed dashed arrows represent confounders affecting both behavioral policy, state transitions and rewards while being unobservable to the online agents.}
\Description{Causal diagrams showing the CMDP for offline data and online learning.}
\end{figure*}

\begin{definition}
\label{def:cmdp}
    A Confounded Markov Decision Process (CMDP) $\mathcal{M}$ is a tuple of $\langle \mathcal{S}, \mathcal{X}, \mathcal{Y}, \mathcal{U}, \3F, \3P \rangle$ where,
    \begin{itemize}[label=\textbullet, leftmargin=20pt, topsep=0pt, parsep=0pt, itemsep=1pt]
        \item $\mathcal{S}, \mathcal{X}, \mathcal{Y}$ are, respectively, the space of observed states, actions, and rewards;
        \item $\mathcal{U}$ is the space of unobserved exogenous noise;
        \item $\3F$ is a set consisting of the transition function $\tau: \1S \times \1X \times \1U \mapsto \1S$, behavioral policy $\beta: \1S \times \1U \mapsto \1X$, and reward function $r: \1S \times \1X \times \1U \mapsto \1Y$;
        \item $\3P$ is a set of distributions $P$ over the unobserved domain $\1U$.
    \end{itemize}
\end{definition}

To model general continuous control problems, we assume the space of states, $\mathcal{S}$, actions, $\mathcal{X}$, and unobserved exogenous noise, $\1 U$, to be multi-dimensional and continuous throughout the paper. 
Consider a demonstrator agent interacting with a CMDP. For every time step $h = 1, \dots$, the nature draws an exogenous noise $U_h$ from the distribution $P(\1U)$; the demonstrator performs an action $X_h \gets f_X(S_h, U_h)$, receives a subsequent reward $Y_h \gets f_Y(S_h, X_h, U_h)$, and moves to the next state $S_{h + 1} \gets f_S(S_h, X_h, U_h)$. The observed trajectories of the demonstrator (from the learner's perspective) are thus summarized as the observational distribution $P(\bar{\*X}, \bar{\*S}, \bar{\*Y})$.
\footnote{We will consistently use $\bar{\*X}, \bar{\*S}, \bar{\*Y}$ to represent trajectory sequences. 
} 
In the data-generating process described above, for every time step $h$, the exogenous noise $U_h$ becomes an unobserved confounder affecting the action $X_h$, reward $Y_h$, and next state $S_{h+1}$ simultaneously. Therefore, CMDP is also referred to MDP with Unobserved Confounders (MDPUC, \citep{DBLP:conf/clear2/ZhangB22}) and is a subclass of Confounded Partially Observed MDP~\citep{DBLP:conf/icml/ShiUHJ22,DBLP:conf/nips/MiaoQZ22,DBLP:journals/ior/BennettK24} where Markov property holds.

The observed distribution of such an offline collected dataset with finite trajectories up to $H$ steps can be written as,
\begin{align}
    \begin{split}
        P(\bar{\*X}_{1:H}, \bar{\*S}_{1:H}, \bar{\*Y}_{1:H}) = P(s_1)\prod^H_{h=1}\bigg(\int_{\1{U}} \I_{s_{h+1} = f_S(s_h, x_h, u_h)} \\
    \I_{x_{h} = f_X(s_h, u_h)} \I_{y_{h} = f_Y(s_h, x_h, u_h)} P(u_h) du_h\bigg)
    \end{split}
\end{align}
The causal diagram in \cref{fig:cmdp offline} showcases how the confounders are affecting state transitions, reward, and the behavioral agent's policy while the online learning agent doesn't have such information when acting in the environment~\cref{fig:cmdp online}. By convention \citep{pearl2009causality}, we use bi-directed arrows (e.g., $X_h \leftrightarrow S_h$) to indicate the presence of unobserved confounders, $U_h$, affecting actions, states and rewards. 

During the online learning phase, as shown in the causal diagram in \cref{fig:cmdp online}, the agent intervenes on the action variable following a policy $\pi(x_h|s_h)$ that maps from state to a distribution over the action domain $\1 X$. This is denoted as policy intervention $\doo(\pi)$ replacing the behavioral policy $f_X$ during the offline data collection phase. The online trajectory distribution in CMDP under $\pi$, $\1M_\pi$ is,
\begin{align}
    \begin{split}
        &P_\pi(\bar{\*X}_{1:H}, \bar{\*S}_{1:H}, \bar{\*Y}_{1:H}) = \\
        &\phantom{xxxxxx}P(s_1)\prod^H_{h=1}\bigg(\pi(x_h|s_h)\1T(s_h, x_h, s_{h+1}) \1 R(s_h, x_h, y_{h})\bigg)
    \end{split}
\end{align}
where the transition distribution $\1 T$ and the reward distribution $\1 R$ are given by, 
\begin{align}
        \1 T(s_h, x_h, s_{h+1}) = \int_{\1{U}} \I_{s_{h+1} = f_S(s_h, x_h, u_h)} P(u_h) du_h\\
        \1 R(s_h, x_h, y_{h}) = \int_{\1{U}} \I_{y_{h} = f_Y(s_h, x_h, u_h)} P(u_h) du_h
\end{align}
In CMDP, the learning goal is still to find the optimal policy $\pi^*$ that maximizes the cumulative return (often discounted by $\gamma \in (0,1)$). That is, $\pi^* = \arg\max_\pi \mathbb{E}[\sum_{h=1}^\infty \gamma^{h-1} y_h]$. This objective function can be solved iteratively using the Bellman Optimality Equation \citep{putermanMarkovDecisionProcesses1994book}, 
\begin{align}
    \begin{split}
        Q^*(s,x) = \mathbb{E}[Y_h + \gamma\max_{x'} Q^*(S_{h+1}, x') |S_h = s, X_h = x]
    \end{split}
\end{align}
where $Q^*(s,x) = \mathbb{E}[\sum_{t=1}^\infty \gamma^{t-1} y_{h+t}|S_h = s, X_h = x]$ by definition is the optimal state action value function denoting the best return after taking action $x$ in state $s$. 

In the offline to online learning setting, one would envision that learning from offline datasets generated by a competitive policy with a good action space coverage should yield near-optimal online policies \citep{CQL2020Kumar,calql}.
This is indeed true under the MDP definition where there are no unobserved confounders. The state transitions and reward functions can be easily identified by the observation distribution of the offline dataset,
\begin{align}
        \1 T(s_h, x_h, s_{h+1}) &= P(s_{h+1}|s_h, x_h)\\
        \1 R(s_h, x_h, y_{h}) &= P(y_h|s_h, x_h).
\end{align}
When the above identification formula hold, several off-policy algorithms have been proposed to estimate the effect of candidate policies from finite observations \citep{watkins1989learning,watkins1992q,swaminathan2015counterfactual,jiang2015doubly,precup2000eligibility,munos2016safe,CQL2020Kumar,calql}. Together with deep learning, these methods could be further extended to complex domains \citep{mnihHumanlevelControlDeep2015a,schulmanHighDimensionalContinuousControl2016a,schulmanProximalPolicyOptimization2017,akkaya2019solvingcube,robocook2023Shi}. However, NUC could be fragile in practice and does not necessarily hold due to violations like sensory mismatch in data generation process. In these situations, applying standard off-policy methods may fail to converge to optimal, despite using powerful deep learning models. \footnote{NUC is orthogonal to the issue of action space coverage discussed by the line of research on conservative Q-learning (CQL) \citep{CQL2020Kumar}. As shown in \Cref{fig:hopper_2_example}, CQL learned state values cannot handle the unobserved confounding issue in offline datasets.}

To witness the challenge of confounded continuous control, we instantiate a confounded variant of Hopper~\citep{towers2024gymnasium} where the offline agent has access to the full state observation, a total of 11-dimension vectors while the online agent can observe all of the dimension except the second dimension, the angle of its thigh joint (comparable to the knee in humans). Not observing the position of its thigh joint presents a genuine challenge to the agent given it would be unable to know to how move its thigh joint to generate the power to hop. As shown in \Cref{fig:hopper_2_example} the vanilla SAC~\citep{haarnoja2018soft} agent trained under partial observation cannot recover $1/2$ of the performance of the agent trained with complete state observations. 

These results expose a critical barrier to reliable learning under restricted sensory capabilities: without access to key hidden factors, even state-of-the-art off-policy RL methods will suffer. And adding good quality offline datasets which supposedly contain valuable information on the optimal state values does not help. As demonstrated in \Cref{fig:hopper_2_example}, commonly used algorithms also fall short under the confounded setting due to their fundamental limitation of assuming NUC implicitly followed from the (PO)MDP definition. Then, our central research question is: Can we extract and transfer such contaminated offline knowledge to guide online learning? The answer is yes. In the next section, we propose the causal bellman equation from which automatic reward shaping with confounded offline data is made possible to facilitate online learning. 

\begin{figure}[H]
    \centering
    \includegraphics[width=\columnwidth]{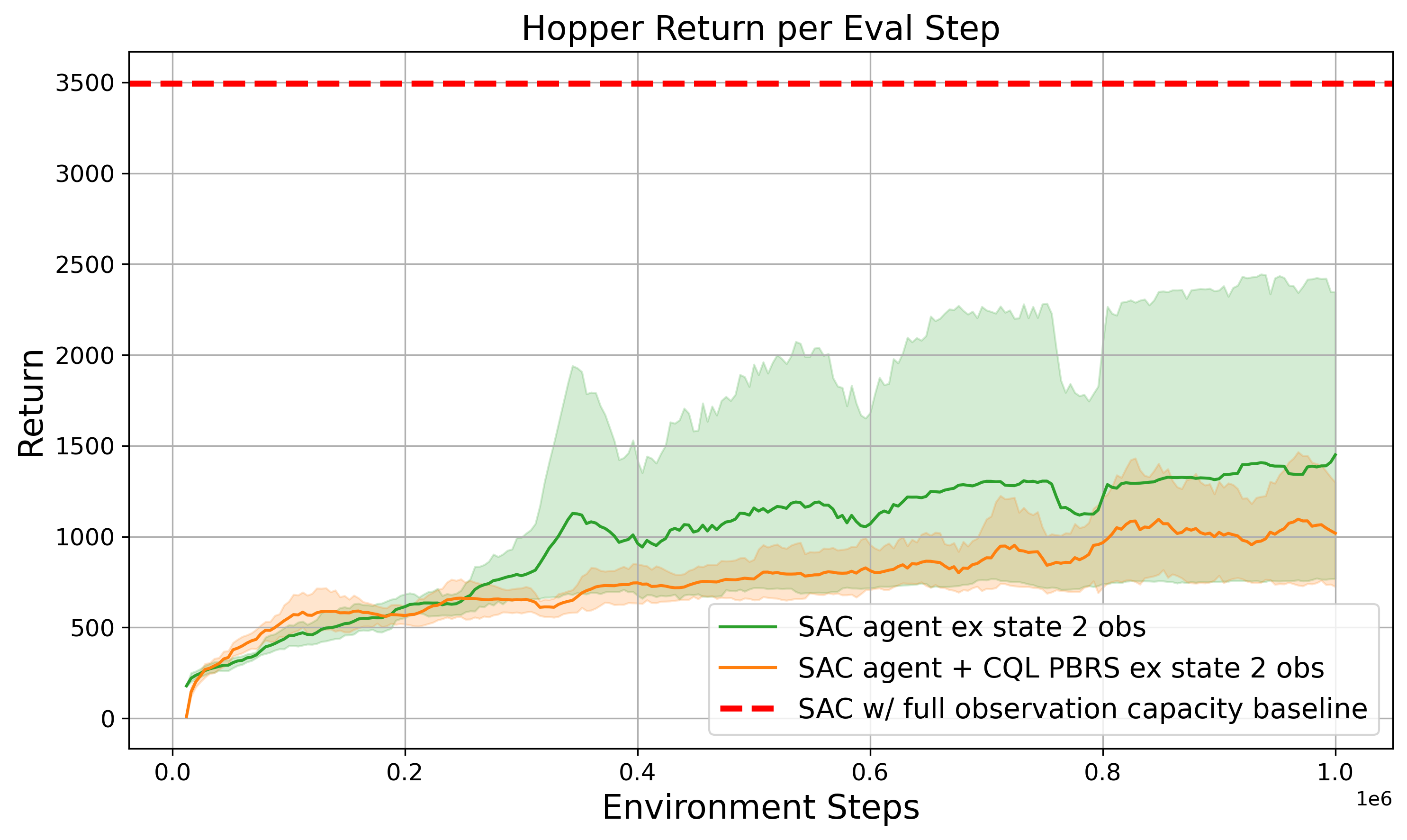} 
    \vspace{-0.1in}
    \caption{Performance of Hopper SAC Agent with full capacity and SAC agent unable to observe state 2.}
    \Description{Performance of Hopper SAC Agent with full capacity and SAC agent unable to observe state 2.}
    \vspace{-0.1in}
    \label{fig:hopper_2_example}
\end{figure}

\section{Confounded Continuous Control with Shaped Rewards}
In this section, we will introduce how we learn an optimistic state potential from confounded continuous offline data, which is then used for online fine-tuning. See \Cref{sec:proof details} in the appendix for proof details of theorems discussed in this section.

\subsection{Learn Optimistic State Potentials via Confounding Robust Offline Pretraining}
It is well acknowledged that a good state potential function is the optimal state value \citep{pbrs}. But without training the agent, one cannot have easy access to the optimal state values. The premise of using state values estimated from offline dataset as state potentials is that such values are close to optimal state values. However, without deliberate control on the quality of the behavioral policy nor the NUC condition, such offline learned values could be heavily biased and cannot be used for reward shaping as we have seen in \Cref{fig:hopper_2_example}. 
Recently \citet{li2025automatic} proposes to use partial identification methods to upper bound the optimal state values for finite horizon non-stationary CMDPs from such confounded offline datasets robustly. 
Here we extend the results to stationary infinite-horizon CMDPs. 
\begin{theorem}[Causal Bellman Optimal Equation for Stationary Infinite-Horizon CMDPs]
\label{thm:causal bellman eq}
    For a CMDP environment $\mathcal{M}$ with reward $Y_h \leq b, b\in 
    \mathbb{R}$, the optimal value of interventional policies, $V^*(\boldsymbol{s}), \forall \boldsymbol{s}\in \1S$, is upper bounded by $V^*(s) \leq \overline{V}(s)$ satisfying the Causal Bellman Optimality Equation,
    \begin{align}
    \begin{split}
        &\overline{V}(s) = \max_x \bigg[P(x|s)  \left( \widetilde{\mathcal{R}}(s,x) +  \gamma\mathbb{E}_{\widetilde{\mathcal{T}}} [\overline{V}(s')] \right) \\
        &\phantom{xxxxxxxxxxxxxx}+ {P(\neg x|s) \Big( b + \gamma\max_{s'}\overline{V}(s') \Big)} \bigg]\label{eq:causal bellman}
    \end{split}
    \end{align}
    where $\widetilde{\mathcal{R}}$ is offline estimated reward distribution and $\widetilde{\mathcal{T}}$   is the estimated transition distribution.
\end{theorem}

Compared with the original Bellman Optimal Equation, the Causal Bellman Equation accounts for the uncertainty brought by confounders in the offline dataset via an extra term, 
\begin{align}    
\big(b+\gamma\max_{s'}\overline{V}_{h+1}(s')\big) \label{eq:causal bellman compensation}
\end{align}
which represents the best return that the agent could have achieved from those ``unselected" actions, i.e., $P(\neg x|s)$. With the Causal Bellman Optimal Equation, we can robustly upper bound the optimal state values from a confounded offline dataset generated by CMDP $\mathcal{M}$. Next we show that the extended Causal Bellman Optimal Equation converges to a unique fixed point, which is a valid upper bound on the optimal interventional state values for online agents (\cref{fig:cmdp online}).
\begin{theorem}[Convergence of Causal Bellman Optimal Equation]
\label{thm:causal bellman convergence}
The Causal Bellman Optimality Equation converges to a unique fixed point, which is also an upper bound on the optimal interventional state values under the assumption that $P(s,x)>0, \forall s,x$ in the stationary infinite horizon CMDP $\1M$.
\end{theorem}

The proposed Causal Bellman Equation does not translate directly into a practical algorithm for high dimensional continuous control problems with function approximators, though. The propensity score of actions $P(x|s), P(\neg x|s)$ is ill-defined as any single point in a continuous distribution has zero probability. Furthermore, naively enumerating all the states to calculate $\max_{s'}\overline{V}(s')$ or enumerating actions to get values from Q-values is intractable in continuous state and action space. Thus, we make several practical approximations when implementing the Causal Bellman Equation.

Firstly, we parametrize the causal upper bound state potential to be $V_{\theta_1}(\cdot)$ and its corresponding target network $V_{\theta'_1}(\cdot)$ for smoother learning updates \citep{mnihHumanlevelControlDeep2015a}. We also replace the outside $\max_x$ in \Cref{eq:causal bellman} with an expectation over the observed action distribution in the offline dataset under the condition that the behavioral policy is competitive. 
Then, we restrict the observed policy distribution to be within a tractable class of Gaussian policies as $P_{\theta_2}$. To train this policy distribution, we maximize the likelihood of the observed actions given states in the offline dataset $\1 D$,
\begin{align}
    J(\theta_2) = \mathbb{E}_{\1 D}\Big [ \log P_{\theta_2}(x|s)\Big] \label{eq:obs action distribution obj}
\end{align}
When training $V_{\theta_1}$, we apply the reparametrization trick to make sampling from $P_{\theta_2}$ differentiable,
\begin{align}
    x = f_{\theta_2}(\epsilon, s) \label{eq:obs action distribution reparam}
\end{align}
where $\epsilon$ is an input noise vector sampled from a fixed standard Gaussian distribution. Lastly, instead of calculating the global best possible next state $\max_{s'}\overline{V}(s')$, we aim at finding the best possible nearby states if an action $x' \neq x$ had been taken. Thus, we choose to model the difference distribution between the current state $s$ and next state $s'$ given state action pair $(s,x)$. Similarly, we restrict this distribution to be Gaussian, $P_{\theta_3}$, and the training objective is to maximize the log likelihood of observed state differences in $\1 D$,
\begin{align}
    J(\theta_3) = \mathbb{E}_{\1 D}\Big [ \log P_{\theta_3}\big(\Delta_s\mid s,x\big)\Big] \label{eq:state diff obj}
\end{align}
where $\Delta_s$ is the state difference between $s$ and $s'$. And we apply the reparametrization trick again for differentiable sampling,
\begin{align}
    \Delta_s = f_{\theta_3}(\epsilon, s, x) \label{eq:state diff reparam}.
\end{align}

Now we can approximate the Causal Bellman Equation backup, $\1B$, with parametrized neural network components as,
\begin{align}
    \begin{split}
        &\hat{\1B}\overline{V}_{\theta_1}(s) = 
        \mathbb{E}_{\1D} \bigg[P_{\theta_2}(x|s)  \left( y +  \gamma \overline{V}_{\theta_1}(s') \right) \\
        &\phantom{xxxxxxxxxxxx}+ P_{\theta_2}(x'|s)\Big( \max_{y \in \1 D} y + \gamma\overline{V}_{\theta_1}(s + \Delta_s) \Big) \bigg]\label{eq:causal bellman parametrized}
    \end{split}
\end{align}
where $x' = \arg\max_{x''\sim P_{\theta_2}(\cdot|s), x'' \neq x} V_{\theta_1}(s + \Delta_s'), \Delta_s'\sim P_{\theta_3}(\cdot|s,x')$ is the action that maximize the return of ``the road not taken", and $\Delta_s \sim P_{\theta_3}(\cdot|s,x')$ is the state difference sampled given $s, x'$.
The extra compensation term in \Cref{eq:causal bellman compensation} is approximated by the maximum reward observed in the offline dataset plus the discounted best next state value as if the agent had taken $x'$ and transited to a state $s'$ that is $\Delta_s$ away from $s$.
In implementation, we also approximate the action propensity score with its corresponding Gaussian density to avoid the zero probability issue.
The state potential function $\overline{V}_{\theta_1}$ is then trained to minimize the squared residual error,
\begin{align}
    &J(\theta_1) = \mathbb{E}_{\1D} \bigg[ \frac{1}{2} \Big(\overline{V}_{\theta_1}(s) - \hat{\1B}\overline{V}_{\theta_1'}(s) \Big)^2 \bigg] \label{eq:causal bellman obj}
\end{align}

See \cref{alg:causal-upper-bound} for the full pseudo-code of learning the state potentials from continuous confounded offline data. In the next section, we will illustrate how to use this learned state potentials as reward shaping functions during online training with SAC.

\begin{algorithm}[t]
\caption{Neural Causal Upper Bound State Potential}
\label{alg:causal-upper-bound}
\begin{algorithmic}[1]
    \State Initialize parameters $\theta_1, \theta_1', \theta_2, \theta_3$
    \While{Not Converged}
        \State Sample an offline batch $\{(s_i, x_i, s'_i, y_i)\}_{i=1}^B$
        \State Update observed policy, $\theta_2 \gets \theta_2 + \lambda_2 \hat\nabla J(\theta_2)$ (\Cref{eq:obs action distribution obj})
        \State Update state difference, $\theta_3 \gets \theta_3 + \lambda_3 \hat\nabla J(\theta_3)$ (\Cref{eq:state diff obj})
    \EndWhile
    \While{Not Converged}
        \State Sample an offline batch $\{(s_i, x_i, s'_i, y_i)\}_{i=1}^B$
        \State Update state potential, $\theta_1 \gets \theta_1 - \lambda_1 \hat\nabla J(\theta_1)$ (\Cref{eq:causal bellman obj})
        \State Update target networks, $\theta_1' \gets \tau \theta_1 + (1-\tau)\theta_1'$
    \EndWhile\\
    \Return $V_{\theta_1}(s)$
\end{algorithmic}
\end{algorithm}

\subsection{Online Fine-tuning with Reward Shaping}
We first re-establish the optimal policy invariance property of PBRS~\citep{pbrs} in infinite-horizon stationary CMDPs.
\begin{proposition}
\label{prop:pbrs cmdp}
    For a CMDP under policy $\pi$, $\1M_{\pi}$, let $\1M_{\pi}'$ be the CMDP obtained from $\1M_{\pi}$ by replacing the reward with the following function, for every time step $h = 1, \dots, H$,
    \begin{align}
        y'_h:=
            y_h + \gamma \phi_h(S_{h+1}) - \phi_h(S_h), \label{eq:shaping}
    \end{align}
    where $y_h$ is the original reward returned by $\1M_\pi$; $\phi_h(\cdot):\mathcal{S}\mapsto \mathbb{R}$ is a real valued potential function. Then every optimal policy in $\mathcal{M}_\pi$ will also be an optimal policy in $\mathcal{M}_\pi'$, and vice versa.
\end{proposition}
This guarantees that the agent could still obtain the same optimal policy as the agent without using reward shaping.

For the learning algorithm, we choose to use a popular value-based maximum entropy learner, SAC \cite{haarnoja2018soft}, for its amenity to PBRS. Because for policy gradient based methods like PPO \cite{schulmanProximalPolicyOptimization2017}, multi-step returns are commonly adopted when calculating advantages \cite{schulmanHighDimensionalContinuousControl2016a} resulting in the intermediate shaping terms being canceled out without affecting learning process. 

The original SAC algorithm learns a soft policy that minimizes the following objective function:
\begin{align}
    J(\phi) = \mathbb{E}_{\1D}\bigg [ {\operatorname{D_{KL}}} \Big(\pi_\phi(\cdot|s) \Big\| \frac{\operatorname{exp}(Q_\theta(s, \cdot))}{Z_\theta(s)}\Big) \bigg]
\end{align}
where the parametrized policy $\pi_\phi$ is updated towards the exponential of the learned Q function $Q_\theta$ and $Z_\theta(s)$ normalizes the distribution.
With the reward shaping function defined in \Cref{eq:shaping}, the Q value function $Q_\theta$ is now trained to track returns with the shaped reward $y'$ instead of $y$ while other parts of SAC stay unchanged.
\section{Experiments}

In this section, we evaluate the efficacy of the state value upper bounds learned in \Cref{alg:causal-upper-bound} by applying them as potentials in Potential-Based Reward Shaping (PBRS, \citep{pbrs}) to augment an SAC agent's training rewards in a suite of confounded, continuous control environments. We compare our results to 1) a baseline online SAC agent that does not use PBRS, 2) CQL PBRS \cite{CQL2020Kumar} that uses offline learned conservative Q-values as shaping potentials in PBRS, 3) an online SAC agent using recurrent networks, and 4) T-REX PBRS \cite{brown2019extrapolating} which uses inverse reinforcement learning to learn a reward function, which we use as the potential. We provide further commentary on which cases the causal reward shaping function can provide the most improvement to online training, and how offline data quality affects performance.

\begin{figure*}[ht]
    \vspace{-0.1in}
    \centering
    \includegraphics[width=.9\textwidth]{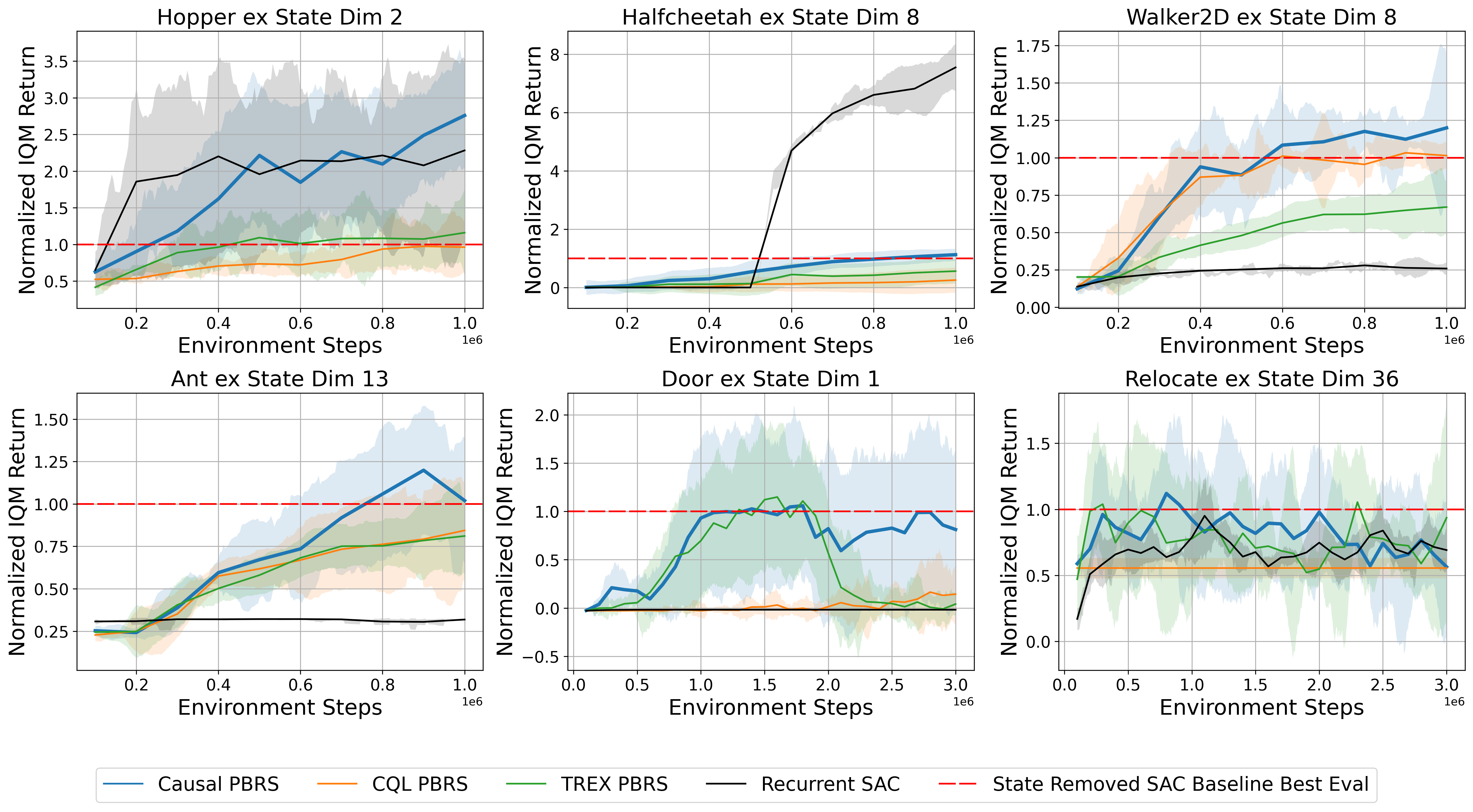} 
    \vspace{-0.1in}
    \caption{Normalized IQM returns w.r.t State Removed Baseline SAC agent in confounded continuous control benchmarks.
    }
    \vspace{-0.1in}
    \Description{Causal PBRS performance by offline data quality.}
    \label{fig:IQM}
\end{figure*}

\subsection{Experiment Design}
\label{experiment_design_section}

\textbf{Environments and Offline Data}: We selected six continuous control tasks from the Gymnasium \cite{towers2024gymnasium} environment to evaluate our causal PBRS method. We selected four MuJoCo \cite{todorov2012mujoco} tasks on robotic locomotion including 1) Hopper - control a single legged agent to hop in 2D space without falling, 2) HalfCheetah - control a two legged, Cheetah-like, agent to run forward in 2D space as fast as possible, 3) Walker2D - control a bipedal agent to walk forward in 2D space without falling over, and 4) Ant - control a quadruped robot to walk forward in 3D space. We also selected two tasks from Adroit \cite{rajeswaran2017learningadroit} on controling a robotic hand to open a door (Door) and to move a ball to the target location (Relocate).  We use the offline datasets from Minari \cite{minari}. For details, see \Cref{sec:env and offline dataset} in the appendix.

\textbf{Confounded Environment Setup}: To simulate unobserved confounders in our experiments, we remove dimensions from the environment's observation space\footnote{See \cref{sec:proving_confounding} for tests showing this method creates confounding bias between the behavior policy and state transition as depicted in \cref{fig:cmdp offline}}. As a result, we simulate trajectories generated by an agent with access to richer sensory information (the removed dimensions), which is not accessible to the online agent. Consequently, typical off-policy learners cannot identify the behavioral policy or the state / q-value function accurately, given the unobserved confounders in the offline datasets \citep{pearl2009causality,zhang2019near}. In practice, removing some of the dimensions will render the task much more challenging than removing others. For example, removing the agent's velocity observation in the MuJoCo environments - which is a key input into the environment's reward function - would make it difficult to derive the useful state value function from the offline data. On the other hand, removing an observation dimension that is independent of the state / Q-value function would not affect standard off-policy learners significantly. Therefore, to understand which dimensions to remove to create meaningful unobserved confounders, we use the Randomized Conditional Independence Test (RCIT, \cite{RCIT}) to measure the independence degree between an observation dimension and the episode's reward-to-go, conditioned on the remaining observations and actions. For detailed analysis on the relations between the removal of an observation dimension and the online agent's performance, see \Cref{ci_v_performance}.

\textbf{Offline and Online Training}: Once we remove selected state dimensions from the offline dataset, we train our state value upper bounds following \Cref{alg:causal-upper-bound}. We train the environment models for 50 epochs, and then train the state value upper bounds for 200 epochs. To prevent over-estimation (in particular from the "road not taken" estimate $\gamma\overline{V}_{\theta_1}(s + \Delta_s)$), we clip state value functions using $\frac{\max_{y \in D} y}{1-\gamma}$, the theoretical max value of the value function given the max observed reward in the offline data. As a comparison to our causal shaping functions, we also estimate Q-value functions using CQL \cite{CQL2020Kumar}, and a reward function using T-REX \cite{brown2019extrapolating}. 

After training the offline state value functions, we test five online learners: 1) a baseline SAC where we remove the selected state dimension from the observation space (SAC Baseline), 2) a PBRS-based approach where we use the learned Q-value function from CQL (CQL PBRS), 3) a PBRS-based approach using T-REX (T-REX PBRS), 4) an SAC agent using recurrent networks (Recurrent SAC) and 5) a PBRS-based approach where the potentials are the learned state value upper bounds (Causal PBRS, ours). For the Causal, T-REX and CQL PBRS method, we apply a scaled version of the shaping rewards detailed in Prop. \ref{prop:pbrs cmdp}. Specifically, we use $y_{env} + \beta(\gamma \phi(s') -  \phi(s))$ as the shaping reward, where $y_{env}$ is the observed environment reward, $\beta$ is a scaling factor, $\phi$ is the potential function, and $s$ is the current state and $s'$ is the next state. Empirically, we find the agent benefits the most from the potentials when $\beta \leq 1$. 

\subsection{Causal Reward Shaping Performance}

\begin{table*}[t]
\centering
\caption{Best evaluation returns of agents ($\pm$ 1 standard deviation) on the 18 confounded environments. 1M steps for MuJoCo, 3M for Adroit).
All results are averaged over 5 seeds (except Recurrent SAC, which is over 3 seeds), and smoothed over 10 evals. Bold Numbers indicate best best-performing methods. Our Causal PBRS method significantly outperforms the baselines.}

\label{tab:state_removal_results}
\adjustbox{max width=\textwidth}{
\begin{tabular}{l|c|c|c|c|c|c|c}
\toprule
Environment & Full State SAC & State Dim Removed & SAC Baseline & CQL & T-REX & Recurrent SAC & Causal PBRS (ours) \\
\midrule

\multirow{8}{*}{Hopper-v5} & \multirow{8}{*}{3500} 
 & 1 & 2073 $\pm$ 1424 & 1778 $\pm$ 1045 & 1276 $\pm$ 835 & 562 $\pm$ 72 & \textbf{2252 $\pm$ 1081} \\
 &  & 2 & 1450 $\pm$ 1123 & 1096 $\pm$ 481 & 1344 $\pm$ 628 & 2432 $\pm$ 1067 & \textbf{2762 $\pm$ 836} \\
 &  & 3 & 3283 $\pm$ 127 & 1478 $\pm$ 444 & 2780 $\pm$ 859 & 1601 $\pm$ 1548 & \textbf{3582 $\pm$ 37} \\
 &  & 5 & \textbf{3321 $\pm$ 136} & 1038 $\pm$ 19 & 2862 $\pm$ 838 & 2878 $\pm$ 573 & 3183 $\pm$ 518 \\
 &  & 7 & 2935 $\pm$ 467 & 1245 $\pm$ 291 & 2550 $\pm$ 976 & 2693 $\pm$ 1147 & \textbf{3011 $\pm$ 332} \\
 &  & 5,6 & 655 $\pm$ 155 & 1143 $\pm$ 728 & 823 $\pm$ 571 & \textbf{1978 $\pm$ 1347} & 1037 $\pm$ 210 \\
 &  & 1,2,3,4 & 304 $\pm$ 30 & \textbf{557 $\pm$ 244} & 310 $\pm$ 67 & 460 $\pm$ 95 & 305 $\pm$ 15 \\
 &  & 7,8,9,10 & 372 $\pm$ 19 & 422 $\pm$ 43 & 367 $\pm$ 62 & \textbf{3468 $\pm$ 124} & 1442 $\pm$ 1449 \\

\midrule
\multirow{3}{*}{Halfcheetah-v5} & \multirow{3}{*}{12400} 
 & 1 & 2013 $\pm$ 32 & 2039 $\pm$ 31 & 2059 $\pm$ 30 & \textbf{7957 $\pm$ 645} & 2055 $\pm$ 24 \\
 &  & 3 & 8931 $\pm$ 1489 & 9748 $\pm$ 177 & 9120 $\pm$ 882 & \textbf{10057 $\pm$ 555} & 8954 $\pm$ 665 \\
 &  & 8 & 896 $\pm$ 494 & 373 $\pm$ 454 & 615 $\pm$ 443 & \textbf{7939 $\pm$ 851} & 1103 $\pm$ 216 \\

\midrule
\multirow{3}{*}{Walker2D-v5} & \multirow{3}{*}{4050} 
 & 6 & 3981 $\pm$ 154 & 3243 $\pm$ 283 & \textbf{4045 $\pm$ 306} & 2210 $\pm$ 1636 & 3899 $\pm$ 587 \\
 &  & 8 & 3640 $\pm$ 208 & 3766 $\pm$ 259 & 2830 $\pm$ 832 & 1105 $\pm$ 51 & \textbf{4295 $\pm$ 373} \\
 &  & 10 & 3925 $\pm$ 137 & 3600 $\pm$ 428 & \textbf{4417 $\pm$ 445} & 1351 $\pm$ 1285 & 3950 $\pm$ 397 \\

\midrule
\multirow{2}{*}{Ant-v5} & \multirow{2}{*}{4000} 
 & 13 & 3084 $\pm$ 607 & 2792 $\pm$ 1013 & 2830 $\pm$ 832 & 994 $\pm$ 4 & \textbf{3389 $\pm$ 1137} \\
 &  & 16 & \textbf{3750 $\pm$ 896} & 3190 $\pm$ 464 & 3662 $\pm$ 1693 & 1225 $\pm$ 448 & 2963 $\pm$ 1369 \\

\midrule
\multirow{1}{*}{AdroitHandDoor-v1} & \multirow{1}{*}{NA} 
 & 1 & 1472 $\pm$ 1385 & 415 $\pm$ 628 & 1659 $\pm$ 1189 & -27 $\pm$ 1 & \textbf{1725 $\pm$ 1396} \\

\midrule
\multirow{1}{*}{AdroitHandRelocate-v1} & \multirow{1}{*}{NA} 
 & 36 & 24 $\pm$ 9 & 13 $\pm$ 2 & 29 $\pm$ 11 & 23 $\pm$ 7 & \textbf{30 $\pm$ 13} \\

\midrule
\multicolumn{3}{l|}{Normalized Mean ($\uparrow$)} & 1.00 & 0.85  & 0.96 & \textbf{1.93} &  1.29 \\
\multicolumn{3}{l|}{Normalized Median ($\uparrow$)} & 1.00 & 0.85  & 0.98 & 0.89 &  \textbf{1.09} \\
\multicolumn{3}{l|}{Normalized IQM ($\uparrow$)} & 1.00 & 0.81  & 0.96 & 0.88 &  \textbf{1.09} \\
\bottomrule
\end{tabular}
}

\end{table*}

\Cref{fig:IQM} shows the interquartile mean (IQM) returns normalized by the state removed SAC baseline's best evaluation score in each envi- ronment. \Cref{tab:state_removal_results} provides the returns (including the best average return reached by the agent and the return at the final evaluation step), normalized mean, median, and inter-quartile mean return relative to the SAC baseline under partial state observation. 

Overall, the Causal PBRS method outperforms the causally unaware reward shaping baselines (CQL, T-REX Shaping) consistently for an average normalized mean score of 1.29 and normalized IQM of 1.09. In particular, the Causal PBRS method performs the best in cases where the removal of a certain state dimension leads to a larger decrease in the SAC baseline performance. For example, the Causal PBRS method outperforms the SAC baseline in the Hopper environment ex. State 1 and 2, where the SAC baseline does not surpass a score of 2000, versus Hopper ex. State 5, where the SAC baseline performance of around 3300 is comparable to the performance of the Hopper with full observational capabilities \cite{haarnoja2018soft}. However, we note there is a limit, given that when removing states 1-4 in the Hopper environment (which are the dimensions related to joint positions), both the Causal PBRS and other methods are unable to learn a meaningful policy. Unsurprisingly, the causally unaware CQL/TREX PBRS methods largely underperforms the SAC baseline (normalized mean of 0.85/0.96), given that the unobserved confounding in the offline reward signal leads to highly biased Q-value estimations. We note the recurrent SAC method is able to overcome the confounding bias for some of the simpler environments (Hopper-v5 and Halfcheetah-v5), although results are inconsistent. For more complex and higher dimensional environments (Ant-v5, Relocate-v1, and Walker-v5), the recurrent SAC baseline is unable to learn a useful policy. In each of the following sections, we provide an overview of some of the observation dimensions removed in each environment and the agent's performance in those environments. For full commentary on all the dimensions removed, see \cref{sec:full_commentary}.

\subsubsection{Hopper}

We removed the following dimensions:

\begin{enumerate}[label=\textbullet, leftmargin=20pt, topsep=0pt, parsep=0pt, itemsep=1pt]
    \item State Dim 1: The Hopper's torso angle. One of the termination conditions in Hopper is if the torso angle is bounded between $[-0.2, 0.2]$. Without this dimension, the agent is unable to know how to adjust its torso angle to maintain healthy body positions. The Causal PBRS can improve on the baseline by around $20\%$, whereas the other methods under perform the baseline by $15\%$ (CQL)  to  $73\%$ (Recurrent SAC).
    \item State Dim 2: Angle of the thigh joint. Removing this dimension has a large impact on the Hopper's performance, reducing its baseline return by $50\%$ vs the Hopper's performance with full obs capacity using SAC. The Causal PBRS method can drive a large improvement of almost $100\%$ vs the baseline. 
    \item State Dim 1,2,3,4: All of the angular positions of the Hopper's joints and torso. Without knowing the position of its joints, the Hopper agent is unable to learn a meaningful policy. The Causal PBRS method is similarly unable to learn a meaningful policy, suggesting a limit to the Causal PBRS's ability to improve performance in highly confounded environments. The CQL/Recurrent SAC method can generate a slightly better policy, although we note the overall return is still low. 
\end{enumerate}

\subsubsection{HalfCheetah}

We removed the following dimensions:
\begin{enumerate}[label=\textbullet, leftmargin=20pt, topsep=0pt, parsep=0pt, itemsep=1pt]
    \item State Dim 1: Angle of the front tip. Without the angle of the front tip (used to orient the Cheetah), the Half Cheetah agent's return drastically decreases from 12,000 to around 2,000. The Recurrent SAC method outperforms all other methods.
    \item State Dim 8: Velocity of the x-coordinate of the front tip. Unlike the Hopper, removing the forward velocity greatly reduces the HalfCheetah's return. The Causal PBRS method achieves a higher return vs the baseline policy, however the Recurrent SAC method outperforms all other methods. 
\end{enumerate}

\subsubsection{Walker2D}

We removed the following dimensions:
\begin{enumerate}[label=\textbullet, leftmargin=20pt, topsep=0pt, parsep=0pt, itemsep=1pt]
    \item State Dim 8: Velocity of the x-coordinate of the torso. Removing the forward velocity dimension of the Walker2d slightly decreases performance. The Causal PBRS method achieves a $20\%$ improvement over the second best method (SAC baseline).   
\end{enumerate}

\begin{figure}[ht]
    \centering
    \includegraphics[width=.9\columnwidth]{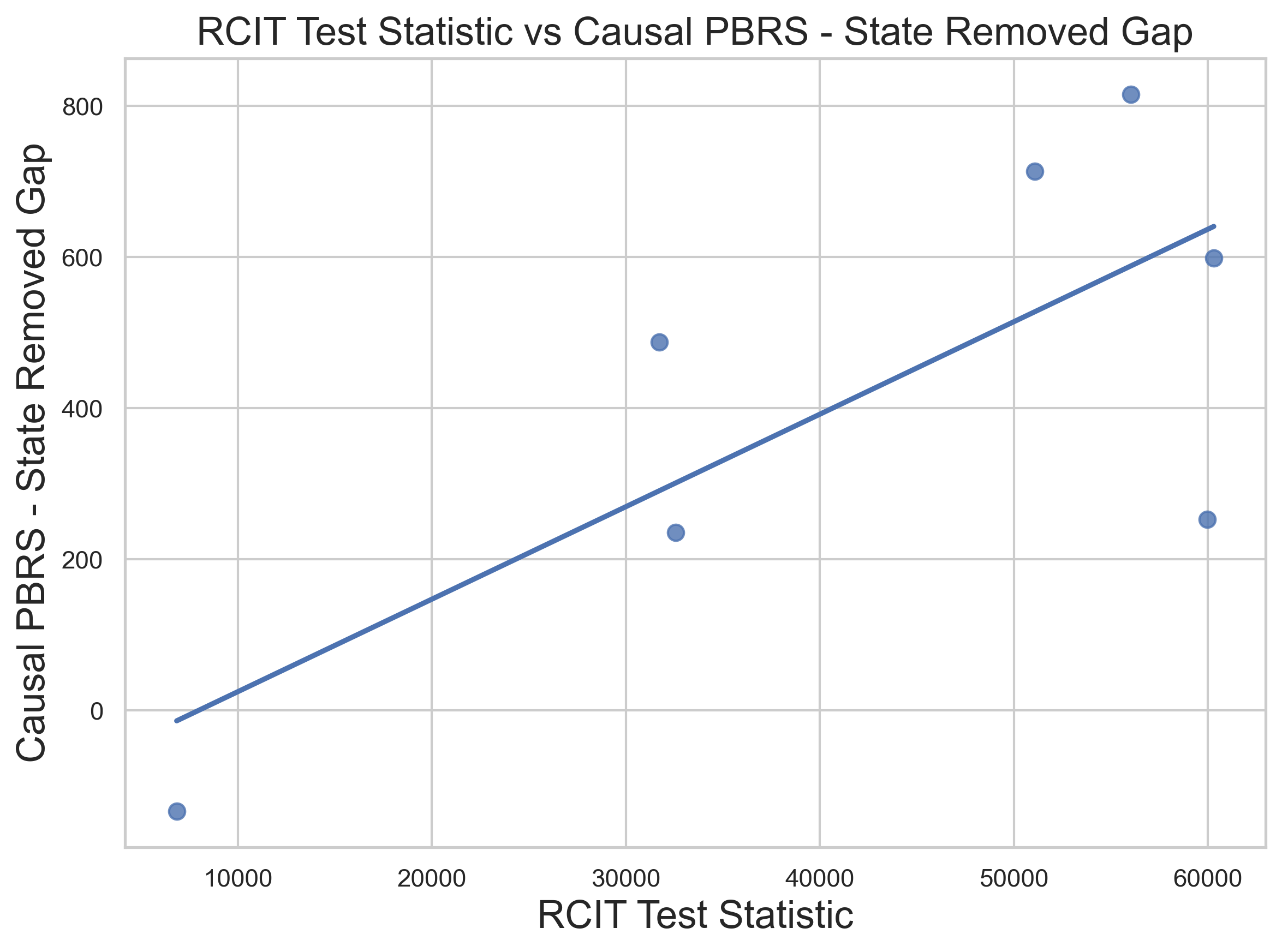} 
    \vspace{-0.1in}
    \caption{RCIT test statistic v.s. Causal PBRS improvements.}
    \Description{Relation Between RCIT test statistic and Causal PBRS improvements.}
    \label{fig:rcit_test}
    \vspace{-0.2in}
\end{figure}

\subsubsection{Ant}

We removed the following dimensions:
\begin{enumerate}[label=\textbullet, leftmargin=20pt, topsep=0pt, parsep=0pt, itemsep=1pt]
    \item State Dim 13: Velocity of the x-coordinate of the torso. Similar to other environments, the Causal PBRS method helps the Ant recover some of the performance loss from the loss of its forward velocity observation, unlike the CQL PBRS method.
     
\end{enumerate}

\subsubsection{Adroit Door}

We removed the following dimensions:
\begin{enumerate}[label=\textbullet, leftmargin=20pt, topsep=0pt, parsep=0pt, itemsep=1pt]
    \item State Dim 1: Angular position of the horizontal arm joint. The baseline agent and Causal PBRS agent perform comparably. The Causal PBRS method outperforms the baseline by $17\%$, and avoids catastrophic forgetting (see \cref{fig:ant_relocate_graphs} for details).
\end{enumerate}

\subsubsection{Adroit Relocate}
We removed the following dimensions:
\begin{enumerate}[label=\textbullet, leftmargin=20pt, topsep=0pt, parsep=0pt, itemsep=1pt]
    \item State Dim 36: x positional difference from the ball to the target. The overall performance is low, but the Causal PBRS method outperforms both the baseline and CQL PBRS method. 

\end{enumerate}

\subsection{Relation Between State Return Independence and Performance Gap}
\label{ci_v_performance}

To understand when the Causal PBRS method can out perform the SAC baseline without shaping, we plot the RCIT statistic from the conditional independence test of the selected state dimension and the remaining returns-to-go (conditioned on the remaining states and actions) from the Hopper-v5 Environment, and the average gap between the Causal PBRS method and the SAC baseline (\Cref{fig:rcit_test}). A higher test statistic implies that the observation dimension has a higher conditional dependence on the returns-to-go. In \Cref{fig:rcit_test}, as the test statistic increases, the gap between the Causal PBRS method and the SAC baseline increases, thus supporting our theory that the Causal PBRS method performs better when there is an increase in confounding bias. We also notice that when we remove dimensions necessary for the task (those with the highest test statistic), the learned causal shaping function may not be informative enough to help recover the performance. See \Cref{sec:further_rcit} for a more details. 

This independence analysis - in addition to exploring the causal model of the MuJuCo environments - suggests the Causal PBRS method performs best when the confounding variables have a direct causal effect on the reward. In the MuJoCo environments, the agent is rewarded for forward movement (known as the forward reward), which is a major determining factor of the total environment reward. In all four of the MuJoCo experiments, when we remove the x dimension velocity observation, a direct parent of the forward reward, the Causal PBRS method is either able to converge faster to an optimal policy or obtain a better policy than the baseline agent. On the other hand, in cases when the Causal PBRS method under performs (such as Ant dimension 16 - the x-coordinate angular velocity of the torso), the the removed dimension by itself does not have as strong an effect on the total reward / behavioral policy.

\subsection{Impact of Offline Data Quality on Online Training with Causal PBRS}

For our experiments, the Minari offline dataset contained three levels of expertise - "simple", "medium", and "expert". The results in \Cref{tab:state_removal_results} are based on training the optimistic state value functions with all three datasets. To evaluate how the expertise of the agent generating the offline data affects the Causal PBRS method, we trained three different causal reward shaping function, each on one of the Minari datasets (  \Cref{fig:offline_data_breakdown}). The Causal PBRS trained solely on the expert data performed the best. However, the Causal PBRS trained with the medium and simple datasets is still able to improve the agent's policy over the baseline. Notably, the performance of the combined datasets (Combined) is comparable to the Expert performance, albeit slightly slower to converge, suggesting \Cref{alg:causal-upper-bound} can converge well in practice, despite the differing data quality. 

\begin{figure}[ht]
    \centering
    \includegraphics[width=.9\columnwidth]{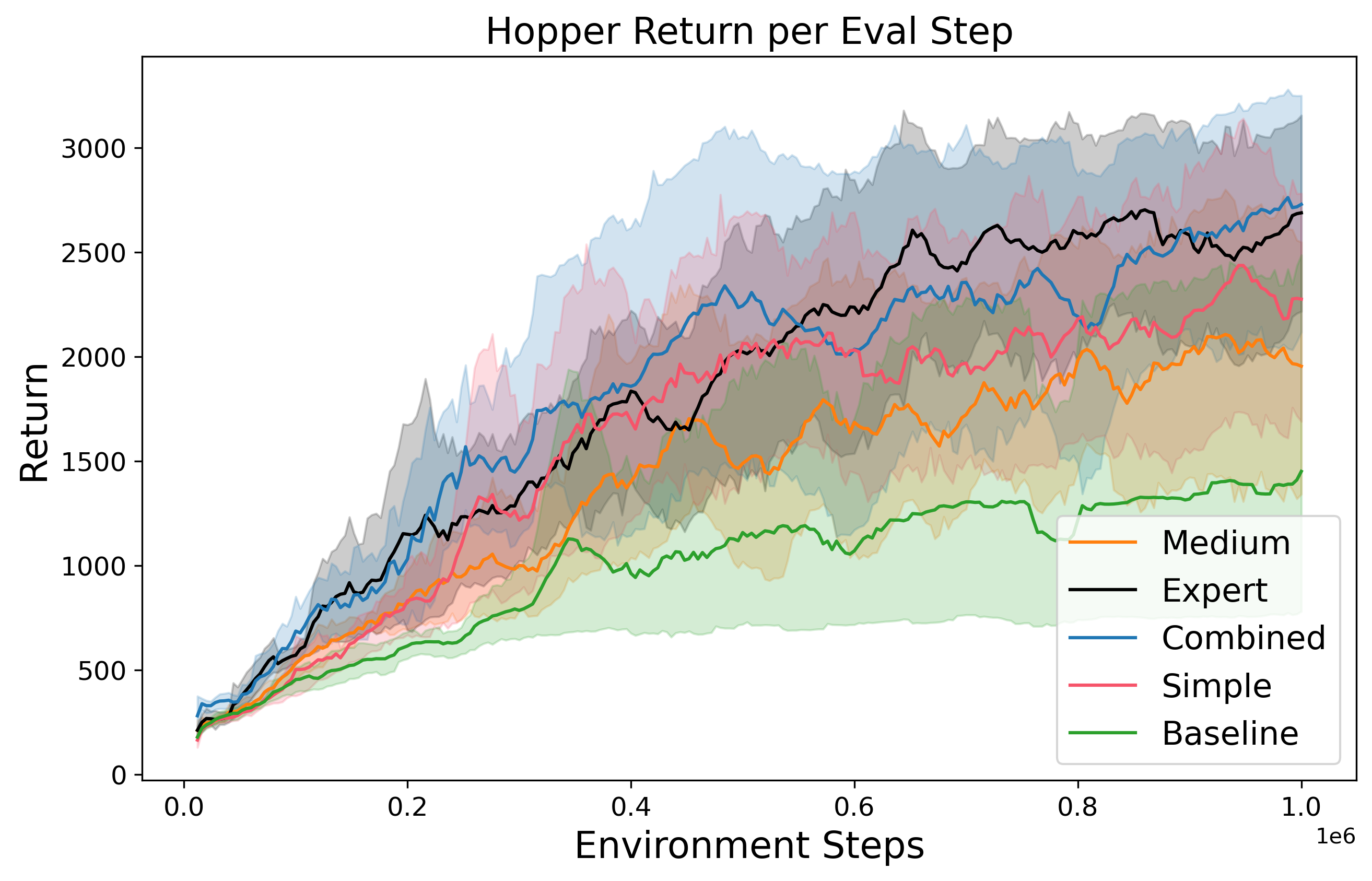} 
    \caption{Causal PBRS performance by offline data quality.}
    \vspace{-0.1in}
    \Description{Causal PBRS performance by offline data quality.}
    \label{fig:offline_data_breakdown}
\end{figure}

\section{Conclusion}
We introduce a causal framework for automatic reward shaping in high-dimensional continuous control with unobserved confounders. By extending the Causal Bellman Equation to continuous settings, our method learns optimistic state potentials that serve as principled shaping functions within the PBRS framework, improving learning efficiency while preserving policy optimality. We evaluate our approach on confounded MuJoCo and Adroit benchmarks, where it consistently outperforms unshaped and causally unaware strong baselines such as CQL, T-REX PBRS from the literature.
This work marks an important step toward confounding-robust reinforcement learning and causal reward design for real-world continuous control. It paves the way for several interesting future directions including: 1) tighter Causal Bellman upper bound given more prior knowledge, 2) testing the Causal PBRS method in more diverse confounding settings and under higher-dimensional image observations, and 3) exploring adaptive Causal PBRS for environments with less confounding bias \footnote{see \cref{sec:future_direct} for an extended discussion.}.

\section*{Acknowledgments}
This research is supported in part by the NSF, ONR, AFOSR, DoE, Amazon, JP Morgan, and The Alfred P. Sloan Foundation.

\clearpage
\bibliography{bibwithID,extra}
\bibliographystyle{ACM-Reference-Format}

\clearpage
\onecolumn
\section*{Appendices}
\label{sec:appendix}
\DoToC
\section{Proof Details}
\label{sec:proof details}
Here we present the proof details of theorems and propositions in the main text.

\subsection{Proof for Thm. \ref{thm:causal bellman eq} Causal Bellman Optimal Equation for Stationary Infinite-Horizon CMDPs}
\begin{proof}
Starting from the Bellman Optimal Equation, the optimal state value function is given by,
    \begin{align}
        V^*(s) &= \max_x R(s,x) + \gamma\sum_{s'} T(s,x,s') V^*(s')
    \end{align}
    Note that the actions here are done by an interventional agent, which is actually $\interv{x}$ in the context of a CMDP. We swap in the causal bounds for interventional reward and transition distribution, 
    \begin{align}
        V^*(s) &\leq \max_x \bigg[\widetilde{R}(s,x)P(x|s) + bP(\neg x|s) + \gamma\sum_{s'} \widetilde{T}(s,x,s')P(x|s)V^*(s') + P(\neg x|s)\max_{s''}V^*(s'')\bigg]
    \end{align}
    where $\widetilde{\mathcal{R}}(s,x) = \mathbb{E}[Y|S=s, X=x]$, $\widetilde{\mathcal{T}}_h$ is shorthand for $\widetilde{\mathcal{T}}(s,x,s') = P(S'=s'|S=s, X=x)$ and $P(x|s) = P(X = x|S = s)$ are estimated from the offline dataset. And $b$ is a known upper bound on the reward signal, $Y \leq b$. In this step, we upper bound the next state transition by assuming the best case that for the action not taken with probability $P(\neg x|s)$, the agent transits with probability $1$ the best possible next state, $\max_{s''}V^*_{}(s'')$. After rearranging terms, we have,
    \begin{align}
        V^*(s) &\leq \max_x \left[ P(x|s) \left( \widetilde{\mathcal{R}}(s,x) + \gamma \sum_{s'} \widetilde{T}(s,x,s'){V}^*(s') \right) + P(\neg x|s) \left( b + \gamma \max_{s''}{V}_{}^*(s'') \right) \right]    
    \end{align}
    And optimizing the value function w.r.t this inequality gives us an upper bound on the optimal state value,
    \begin{align}
        \overline{V}(s) &\leq \max_x \left[ P(x|s)  \left( \widetilde{\mathcal{R}}(s,x) +  \gamma\sum_{s'} \widetilde{T}(s,x,s')\overline{V}_{}(s') \right) + P(\neg x|s) \left( b + \gamma\max_{s''}\overline{V}_{}(s'') \right) \right].\label{eq:bellman update inequality} 
    \end{align}
\end{proof}

\subsection{Proof for Thm. \ref{thm:causal bellman convergence} Convergence of Causal Bellman Optimal Equation}
\begin{proof}
    We will first show that the following Causal Bellman Optimality operator (will denote as ``the operator" or $B$ below for simplicity) is a contraction mapping with respect to a max norm. 
    Then by Banach's fixed-point theorem~\citep{BanachSurLO}, this operator has a unique fixed point, and updating any initial point iteratively will converge to it. Then we show that this unique fixed point is indeed a lower bound of the optimal interventional Q-value.
    
    Let the operator $B$ be,
    \begin{align}
        B\overline{V^{}}(s, x) =  \max_x \bigg[ P(x\mid s)\Big (\widetilde{\1R}\Parens{s, x} + \gamma \sum_{s', x'} \widetilde{\1T}\Parens{s, x, s'} \overline{V}(s') \Big ) 
		+ P(\neg x \mid s) \Big (b + \gamma \max_{s'}\overline{V}(s') \Big ) \bigg].
    \end{align}
    For arbitrary value bounds, $\overline{V^1}, \overline{V^2}$, let their initial difference under max-norm be $c = \left\| \overline{V_{}^1} - \overline{V_{}^2}\right\|_\infty \geq 0$. We can bound their difference after one step update by,
    \begin{align}
        \left\| B\overline{V_{}^1} - B\overline{V_{}^2}\right\|_\infty &\leq 
        \gamma \max_{s,x} \left[P(x|s) \sum_{s'} \widetilde{T}(s,x,s')\left( \overline{V_{}^1}(s') - \overline{V^2}(s')\right) + P(\neg x|s) \max_{s'} \left( \overline{V_{}^1}(s') - \overline{V^2}(s')\right) \right].
    \end{align}
    Thus, under the operator $T$, we have non-expansion Q-value differences,
    \begin{align}        
        \left\| B\overline{V_{}^1} - B\overline{V_{}^2}\right\|_\infty &\leq 
        \gamma \max_{s,x} \left[P(x|s) \sum_{s'} \widetilde{T}(s,x,s')\left( \overline{V_{}^1}(s') - \overline{V^2}(s')\right) + P(\neg x|s) \max_{s'} \left( \overline{V_{}^1}(s') - \overline{V^2}(s')\right) \right]\\
        &\leq 
        \gamma c \max_{s,x} \left( P(x|s) \sum_{s'} \widetilde{T}(s,x,s') + P(\neg x|s)\right),\\
        &=
        \gamma c.
    \end{align}
    for all $\overline{V^1}, \overline{V^2}$ satisfying $ \left\| B\overline{V_{}^1} - B\overline{V_{}^2}\right\|_\infty \leq c, c\geq 0$. Thus, $T$ is a contraction mapping with respect to the max norm. And there exists a unique fixed point $\overline{V^{*}}$ when we apply this operator $B$ iteratively to an arbitrary state value vector till convergence. 
    
    We then show that this fixed point is indeed an upper bound to the optimal interventional state values. By the backup rule of $B$ (\cref{eq:bellman update inequality}), $\forall V(s), V(s)\leq BV(s)$. Thus, for the optimal state value, we can have $V^*(s) \leq \lim_{k\to \infty}B^kV^*(s) = \overline{V^*}(s)$ where $B^k$ denotes applying causal bellman backup $B$ iteratively for $k$ times. This concludes the proof.
\end{proof}

\subsection{Proof for Prop. \ref{prop:pbrs cmdp}}
\begin{proof}
    Because CMDP also enjoys the Markov property, the overall proof procedure highly resembles the original one in \citet{pbrs}. Only to note that the optimal policy invariance is proved in the online learning sense, which is between CMDP $\mathcal{M}'_{\pi}$ after reward shaping under policy $\pi$ and the original CMDP $\mathcal{M}_{\pi}$ under policy $\pi$.
\end{proof}

\section{Environments and Offline Dataset}
\label{sec:env and offline dataset}

\Cref{fig:environments_vis} provides visualization of the six environments tested. 

For offline datasets, the MuJoCo tasks use offline data generated by three different policies of varying expertise (simple-v0, medium-v0, expert-v0). The Adroit tasks use mixed offline data generated by human demonstrators (human-v2), an "expert" fine-tuned RL policy (expert-v2), and an imitation policy of the Human and Expert policies (cloned-v2). 

\begin{figure}[H]
    \centering

    \begin{subfigure}[t]{0.15\textwidth}
        \centering
        \includegraphics[width=\linewidth]{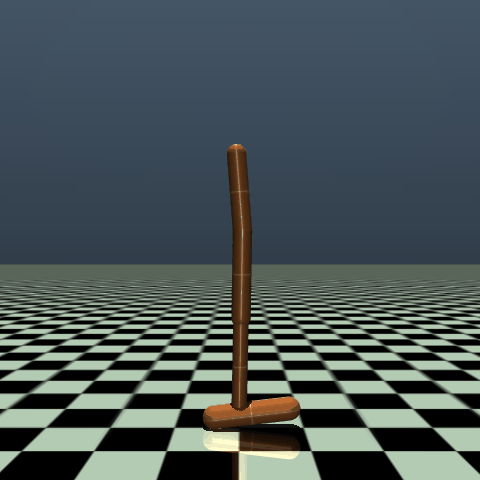}
        \caption{Hopper}
    \end{subfigure}
    \begin{subfigure}[t]{0.15\textwidth}
        \centering
        \includegraphics[width=\linewidth]{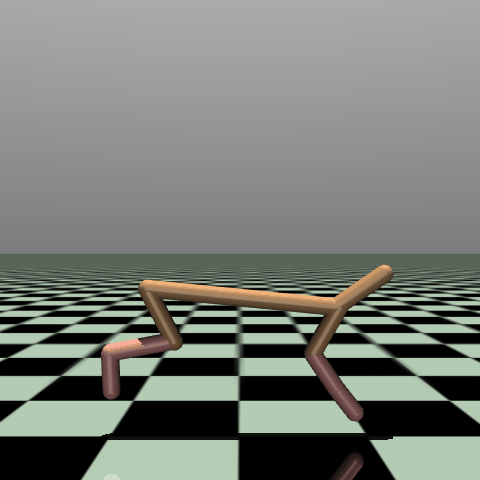}
        \caption{HalfCheetah}
    \end{subfigure}
    \begin{subfigure}[t]{0.15\textwidth}
        \centering
        \includegraphics[width=\linewidth]{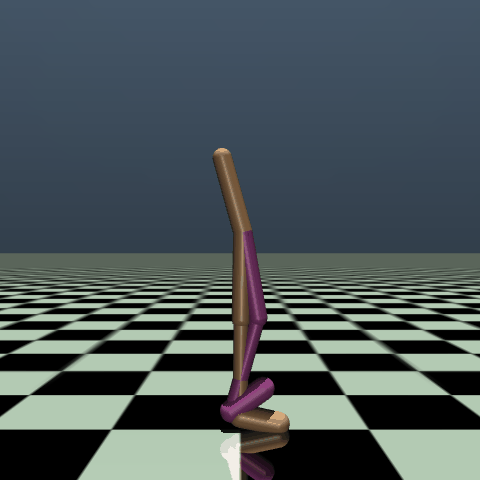}
        \caption{Walker2d}
    \end{subfigure}

    \vspace{0.1cm} 
    
    \begin{subfigure}[t]{0.15\textwidth}
        \centering
        \includegraphics[width=\linewidth]{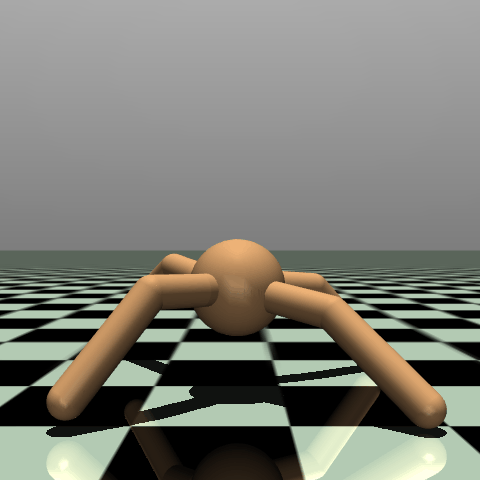}
        \caption{Ant}
    \end{subfigure}
    \begin{subfigure}[t]{0.15\textwidth}
        \centering
        \includegraphics[width=\linewidth]{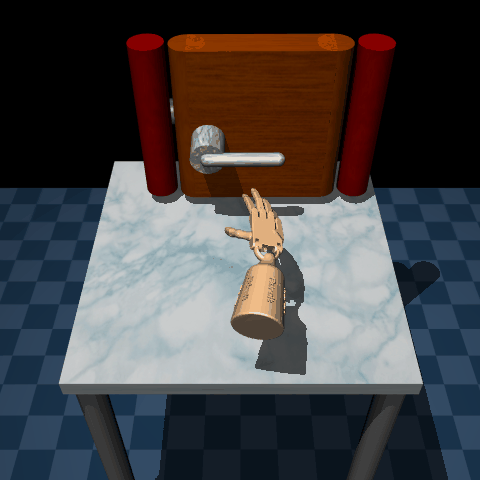}
        \caption{Adroit Door}
    \end{subfigure}
    \begin{subfigure}[t]{0.15\textwidth}
        \centering
        \includegraphics[width=\linewidth]{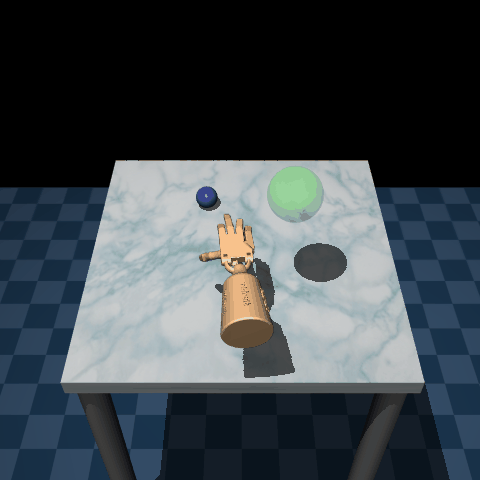}
        \caption{Adroit Relocate}
    \end{subfigure}

    \caption{Visualizations of the Six Environments Tested.}
    \label{fig:environments_vis}
    \Description{Visualizations of the Five Environments Tested.}
\end{figure}

\begin{table}[H]
\centering
\caption{Offline Dataset Size}
\begin{tabular}{l|l}
\toprule
\textbf{Environment} & \textbf{Trajectory Size} \\
\midrule
Hopper-v5 & 2,997,774 \\
HalfCheetah-v5 & 3,000,000 \\
Walker2d-v5 & 2,998,745 \\
Ant-v5 & 3,998,498 \\
AdroitHandDoor-v1 & 2,009,942 \\
AdroitHandRelocate-v1 & 2,006,729 \\

\bottomrule
\end{tabular}
\end{table}

Data pulled from the Minari data repository \citep{minari}. The Hopper, HalfCheetah, and Walker2d offline datasets roughly have the same split of Simple, Medium, and Expert policies. The Ant dataset is $1/2$ Expert data, with the remaining data evenly split between the Simple and Medium policies. The Door and Relocate datasets have 1m observations from the Expert and Cloned policies, with the remaining coming from Human demonstrators. 

\section{Hyper Parameters}

\subsection{Offline Hyper Parameters}

For CQL, we used the baseline implementation and training hyperparameters in \cite{huang2022cleanrl}, trained the Q-value function for 1M timesteps, and used a batch size of 1028. For the Causal Upper Bounded State Value functions, we trained an environment model for 50 epochs to estimate the parameters for ${\theta_2}$ and $\theta_3$ and then trained \Cref{alg:causal-upper-bound} for 200 epochs, with a batch size of 1028. This is roughly equivalent to 600k timesteps in environments with 3M trajectories. To facilitate convergence, we mean-normalized the offline rewards. 

\begin{table}[h]
\centering
\caption{Hyperparameters for Offline Training Causal Upper Bounded State Value Functions}
\begin{tabular}{l|l}
\toprule
\textbf{Hyperparameter} & \textbf{Value} \\
\midrule
Environment Model Training Epochs & 50 \\
Causal Upper Bound Model Training Epochs & 200 \\
Optimizer & Adam \\
Batch Size & 1028 \\
Policy $\theta_2$ Learning Rate & 1e-4 \\
State Transition $\theta_3$ Learning Rate & 1e-5 \\
Q Learning Rate & 1e-4 \\
Discount Factor & 0.99 \\
Target Network Update $\tau$ & 0.005 \\
Target Update Interval & 1 \\
Policy Training Frequency & 3 \\
Num Hidden Dim & 128 \\
Num Residual Blocks & 3 \\

\bottomrule
\end{tabular}
\end{table}

\subsection{Online Hyper Parameters}

\begin{table}[ht]
\centering
\caption{Hyperparameters for Online Training SAC}
\begin{tabular}{l|l}
\toprule
\textbf{Hyperparameter} & \textbf{Value} \\
\midrule
Training Steps & 1e6 (MuJoCo), 3e6 (Adroit)\\
Optimizer & Adam \\
Batch Size & 512 \\
Policy Learning Rate & 3e-4 \\
Q Learning Rate & 1e-3 \\
Alpha & 0.2 \\
Discount Factor & 0.99 \\
PBRS Discount Factor & 1 \\
PBRS $\beta$ & See \Cref{table:beta_hyperparams} \\
Target Network Update $\tau$ & 0.005 \\
Target Update Interval & 1 \\
Policy Training Frequency & 2 \\
Gradient Steps & 1\\
Num Hidden Dim & 256 \\
Num Residual Blocks & 2 \\
Recurrent SAC Batch Size & 16 \\
Recurrent SAC Step Size & 50 \\

\bottomrule
\end{tabular}
\end{table}

To tune the hyperparameter $\beta$ as detailed in \Cref{experiment_design_section}, we tested the following values: 1, 0.1, 0.01, and 0.001, to find the $\beta$ that resulted in the best performance at the overall environment level. \Cref{table:beta_hyperparams} includes the $\beta$ value used for each environment and method. Given that the agent's performance is sensitive to $\beta$, the performance of an agent using Causal or CQL PBRS methods might be improved by further hyper-tuning or optimizing $\beta$ (instead of picking one $\beta$ for the whole environment). 

\begin{table}[ht]
\label{table:beta_hyperparams}
\centering
\caption{\textbf{SAC PBRS Scaling Factors}}
\begin{tabular}{l|c|c|c}
\toprule
\textbf{Environment} & \textbf{Causal Beta} & \textbf{CQL Beta} & \textbf{T-REX Beta} \\
\midrule
Hopper-v5 & 0.1 & 0.1 & 1.0 \\
HalfCheetah-v5 & 0.001 & 0.001 & 1.0 \\
Walker2d-v5 & 0.1 & 0.01 & 1.0 \\
Ant-v5 & 0.001 & 0.001 & 1.0 \\
AdroitHandDoor-v1 & 0.1 & 0.1 & 1.0 \\
AdroitHandRelocate-v1 & 0.1 & 0.1 & 1.0\\

\bottomrule
\end{tabular}
\end{table}

\section{Relation between Conditional Independence and Performance}
\label{sec:further_rcit}
We looked at the relation between the RCIT test statistic (measure of dependence of selected state dimension and returns-to-go, conditioned on remaining states and actions) and performance of the State Removed Baseline and Causal PBRS method. Note, a higher RCIT test statistic implies a higher dependence between the two selected variables, conditioned on the remaining variables. We use results from the Hopper-v5 with the following re dimensions in the Hopper Environment - 1, 2, 3, 5, 6, 7, 1 - 4, 5 - 6, 7 - 10. State Removed Gap is defined as the average eval return gap between the State Removed SAC Baseline and a full state SAC Hopper agent. The Causal Gap is defined as the average eval return gap between the State Removed SAC Baseline and the Causal PBRS method. As seen in \Cref{fig:state_removed_rcit}, there is a negative correlation between an increase in conditional dependence, and Hopper performance degradation. Similarly, as seen in \Cref{fig:rcit_causal_full}, there is a positive relationship between an increase in the RCIT test statistic and gap between the Causal PBRS and State Removed Baseline - but only up until a point. For certain dimensions that have a very high dependence, the causal upper bounded state value function might not be as informative, and hence a less improvement when compared the to State Removed Baseline.

\begin{figure}[H]
    \centering
    \begin{subfigure}[b]{0.48\columnwidth}
        \centering
        \includegraphics[width=\textwidth]{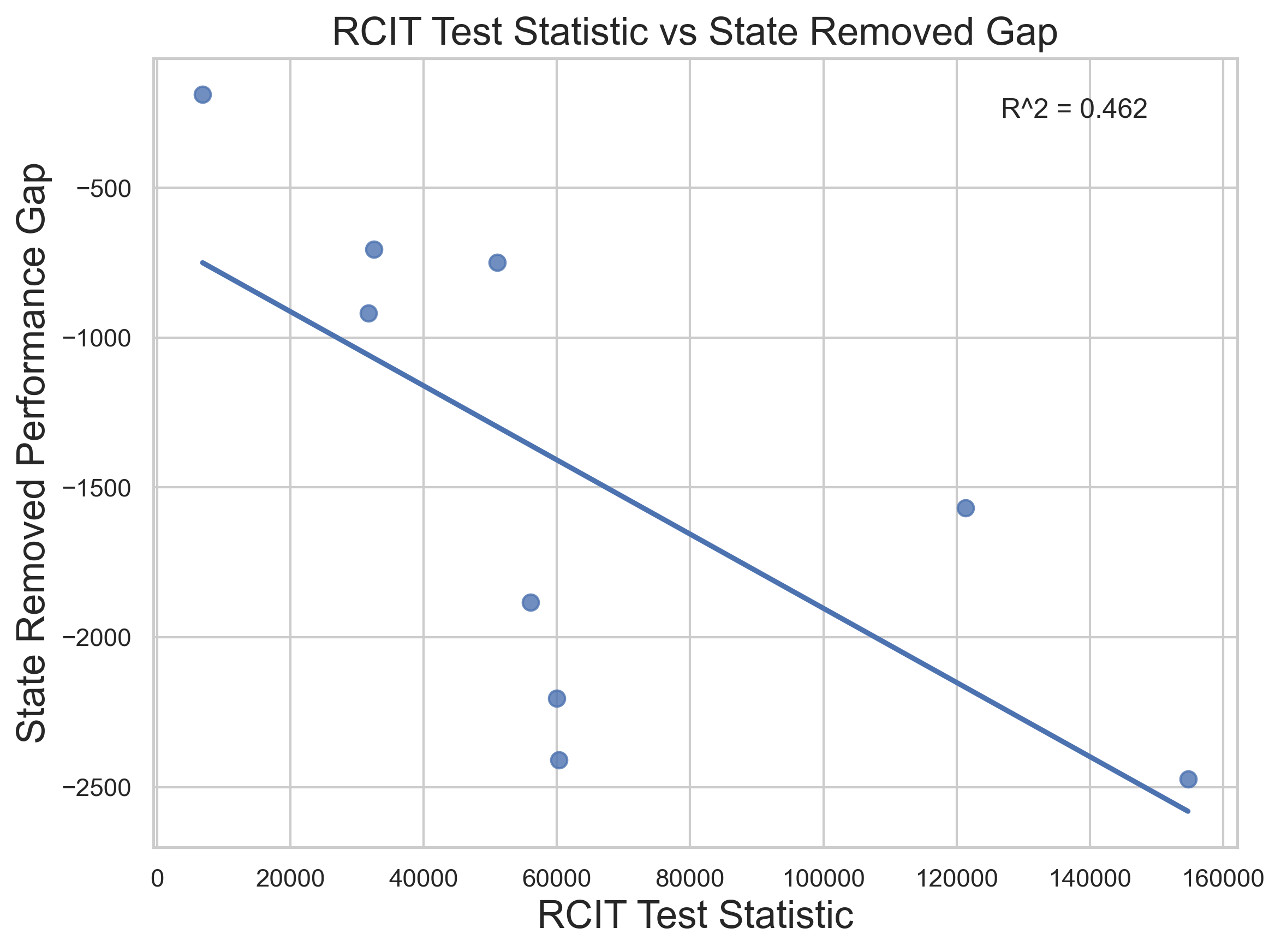}
        \caption{Relation Between RCIT test statistic and State Removed Baseline Degradation in Hopper-v5.}
        
        \label{fig:state_removed_rcit}
    \end{subfigure}
    \hfill
    \begin{subfigure}[b]{0.48\columnwidth}
        \centering
        \includegraphics[width=\textwidth]{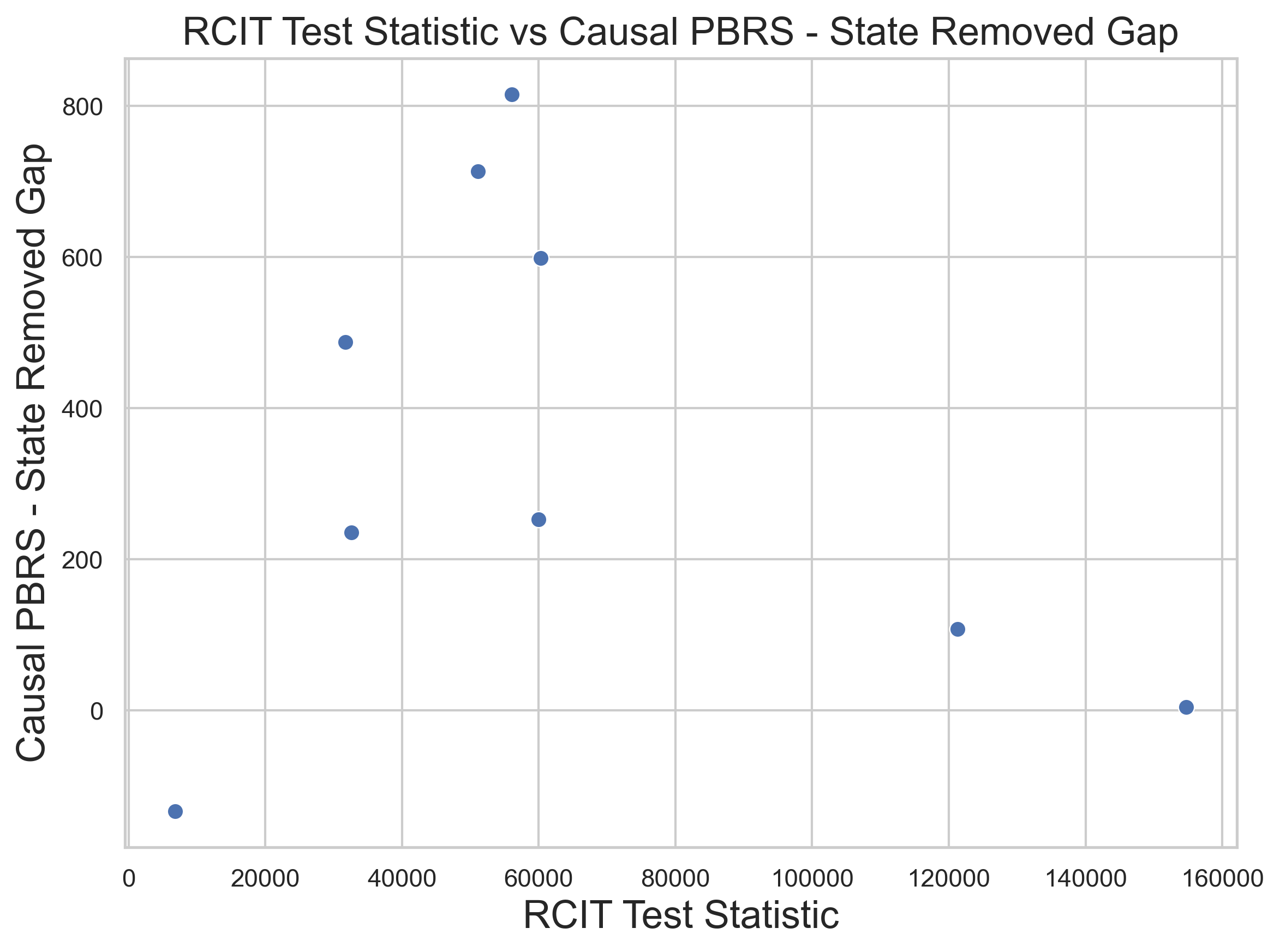}
        \caption{Relation Between RCIT test statistic and Causal PBRS Improvement vs State Removed Baseline.}
        \label{fig:rcit_causal_full}
    \end{subfigure}

    \caption{Relations between RCIT test statistic, baseline degradation, and causal PBRS improvement in Hopper-v5.}
    \Description{Relations between RCIT test statistic, baseline degradation, and causal PBRS improvement in Hopper-v5.}
    \label{fig:rcit_combined}
\end{figure}

\section{Future Directions and Challenges}
\label{sec:future_direct}
This work takes a step towards improving reinforcement learning in confounded, continuous control environments and causal reward design for real-world continuous control. In this section, we discuss current limitations, challenges, and consequently future directions.

First, estimating the value of the "road-not-taken" to adjust for potential confounders in \Cref{thm:causal bellman eq} is not trivial. In particular, deciding which actions $\neg X$ to take is challenging given that in some offline datasets, there may not be enough action policy variance to 1) determine reasonable $\neg X$ actions and 2) understand the state transition probability given state $s$ and action $\neg X$. These challenges lead to the first potential future direction of the paper, that is, exploring better ways to estimate the Causal Bellman equation. One potential direction is to explore additional parametrization techniques to estimate the components of the causal bellman equation (estimating $P(x|s)$, sampling $~X$, estimating the state transition model, etc). Another potential solution is to include datasets generated by a variety of agents with varying capability (\cite{li2025automatic}) for offline dataset collection. Adding and testing the Causal PBRS method's efficacy with a wider range of policy levels could also deepen the understanding of how much expert data is needed to learn effective potentials. 

The second potential direction is testing the Causal PBRS method in more diverse confounding settings and in higher-dimensional image based settings. In this paper, we focused on settings where certain dimensions are masked, as this is straightforward way to simulate unobserved confounders and representative of real world scenarios (where an agent needs to learn from a more sophisticated agent with access to richer sensory dimensions unavailable to the simpler agent). But we acknowledge that there are other types of confounded real world environments the Causal PBRS method could be tested in like environments with strong sensor noises and actuator noises. The other main challenge of this work was selecting the appropriate scaling factor, $\beta$, during the online training phase. Implementing methods to automatically tune or optimize $\beta$ could lead to a large performance improvement in the Causal PBRS method. 

Finally, we observed the Causal PBRS method performs better when confounding biases affect the rewards strongly. Therefore, one last direction is exploring how to learn a tighter Causal PBRS in environments with weaker confounding bias.

\section{Creating Confounded Environments Through Observation Masking}
\label{sec:proving_confounding}

To demonstrate confounding bias in the offline data when masking dimensions between the action and state transition (as depicted in \cref{fig:cmdp offline}), we run two RCIT tests. First, $s^t_h \perp s^{t+1}_{-h} | x^t,s^t_{-h}$ where h is the hidden dimension, and  second, $s^t_h \perp x^t |s^t_{-h}$ for each dimension in Hopper-v5. For all tests, we rejected the null hypothesis (i.e. the hidden states are not conditionally independent with either the behavior policy or state transition). This suggests that the masking method succeeds in creating confounding bias. \Cref{fig:combined_rcit_test_stats} shows the different test statistics of the RCIT test per dimension. As a comparison point, dimension -1 is where we set $s^t_{-h}$ equal to random noise. We see all test statistics are much larger than the random noise test statistic.

\begin{figure}[hbt]
    \centering

    \begin{subfigure}[b]{0.48\textwidth}
        \centering
        \includegraphics[width=\linewidth]{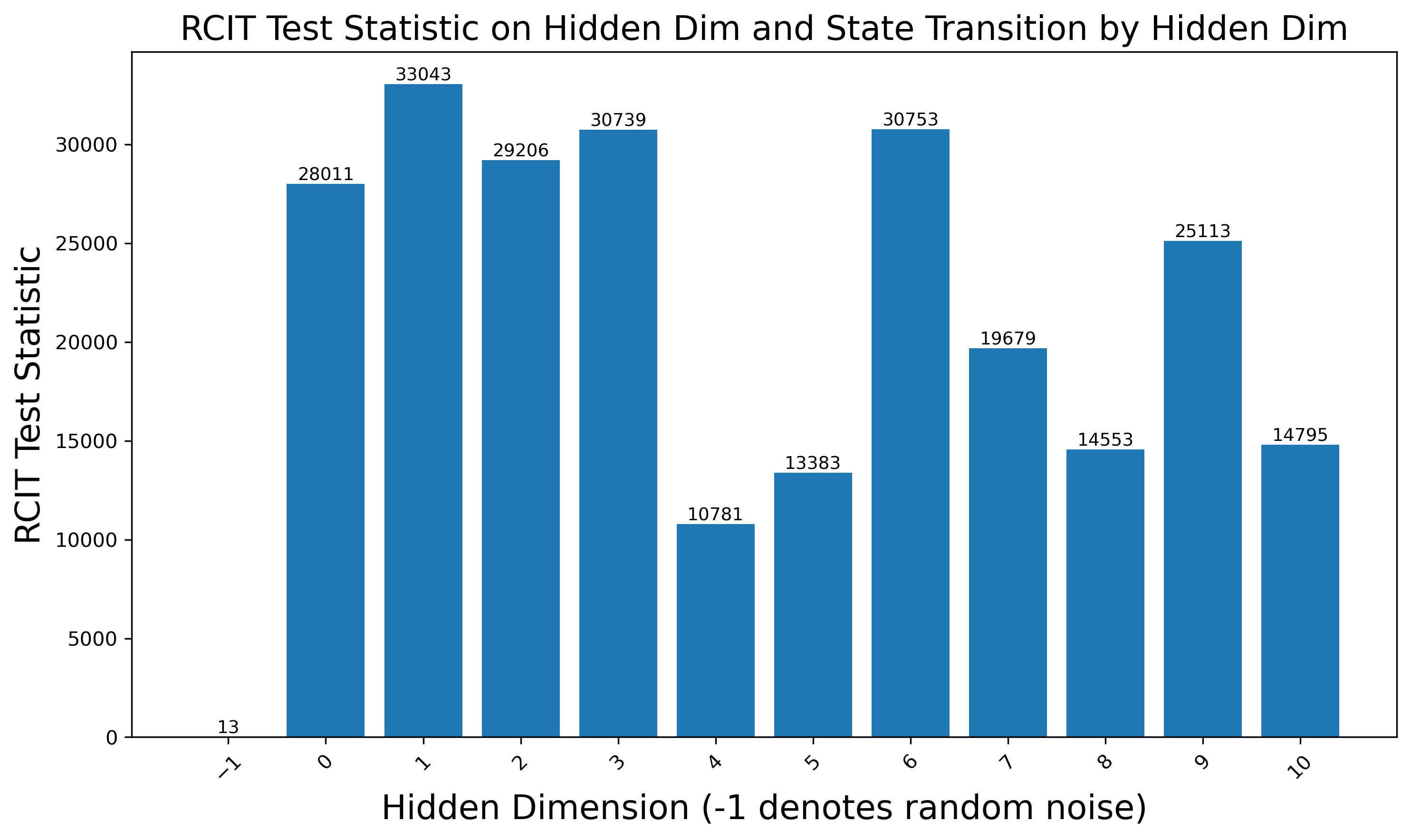}
        \caption{State Transition Independence Tests}
        \label{fig:state_transition_rcit}
    \end{subfigure}
    \hfill
    \begin{subfigure}[b]{0.48\textwidth}
        \centering
        \includegraphics[width=\linewidth]{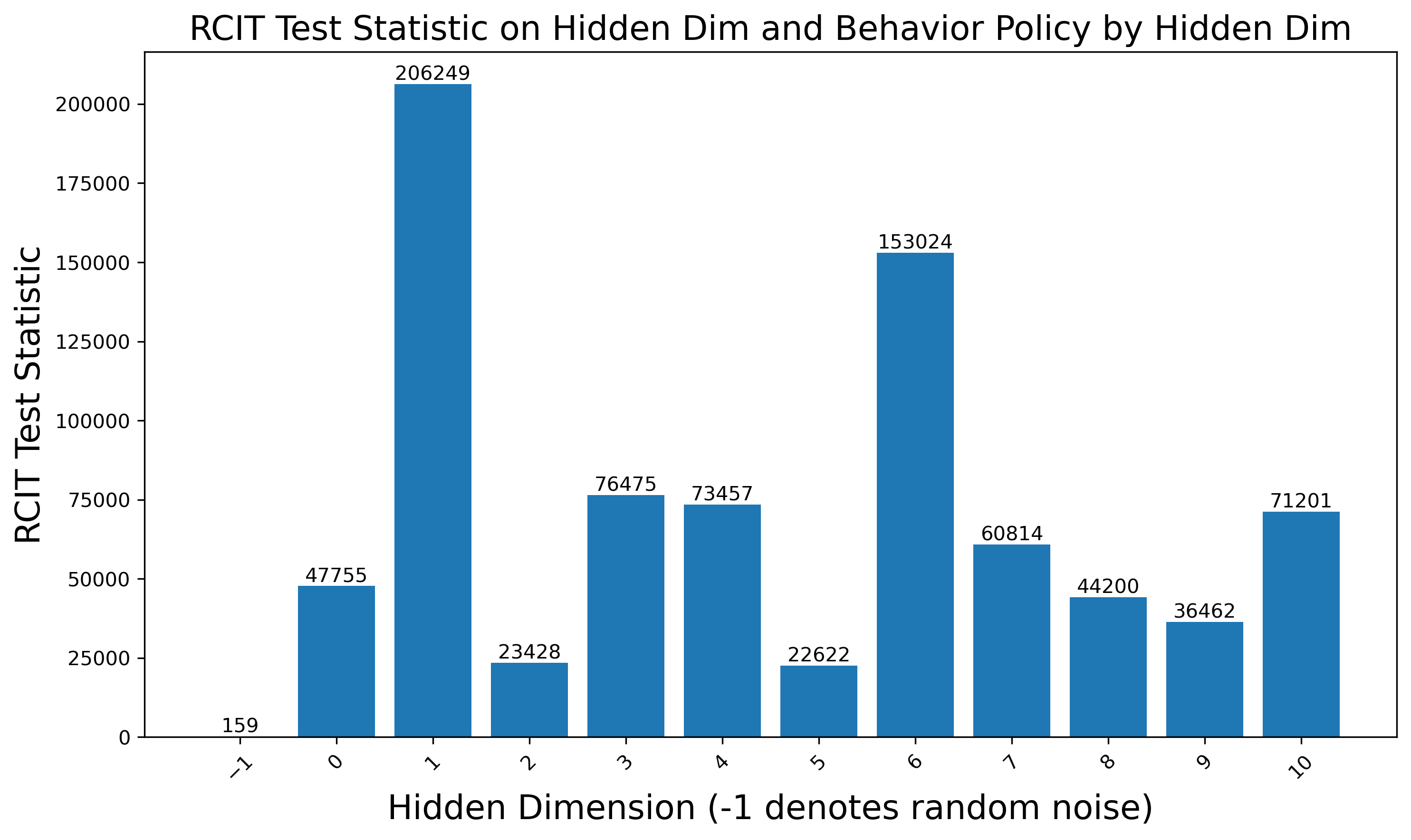}
        \caption{Behavior Policy Independence Tests}
        \label{fig:action_rcit}
    \end{subfigure}

    \caption{RCIT Test Statistics to Test Confounding Bias}
    \label{fig:combined_rcit_test_stats}
\end{figure}

\section{Additional Experiment Results and Tables}
\label{sec:full_graphs_and_tables}

\begin{figure}[H]
    \centering
    \includegraphics[width=.9\textwidth]{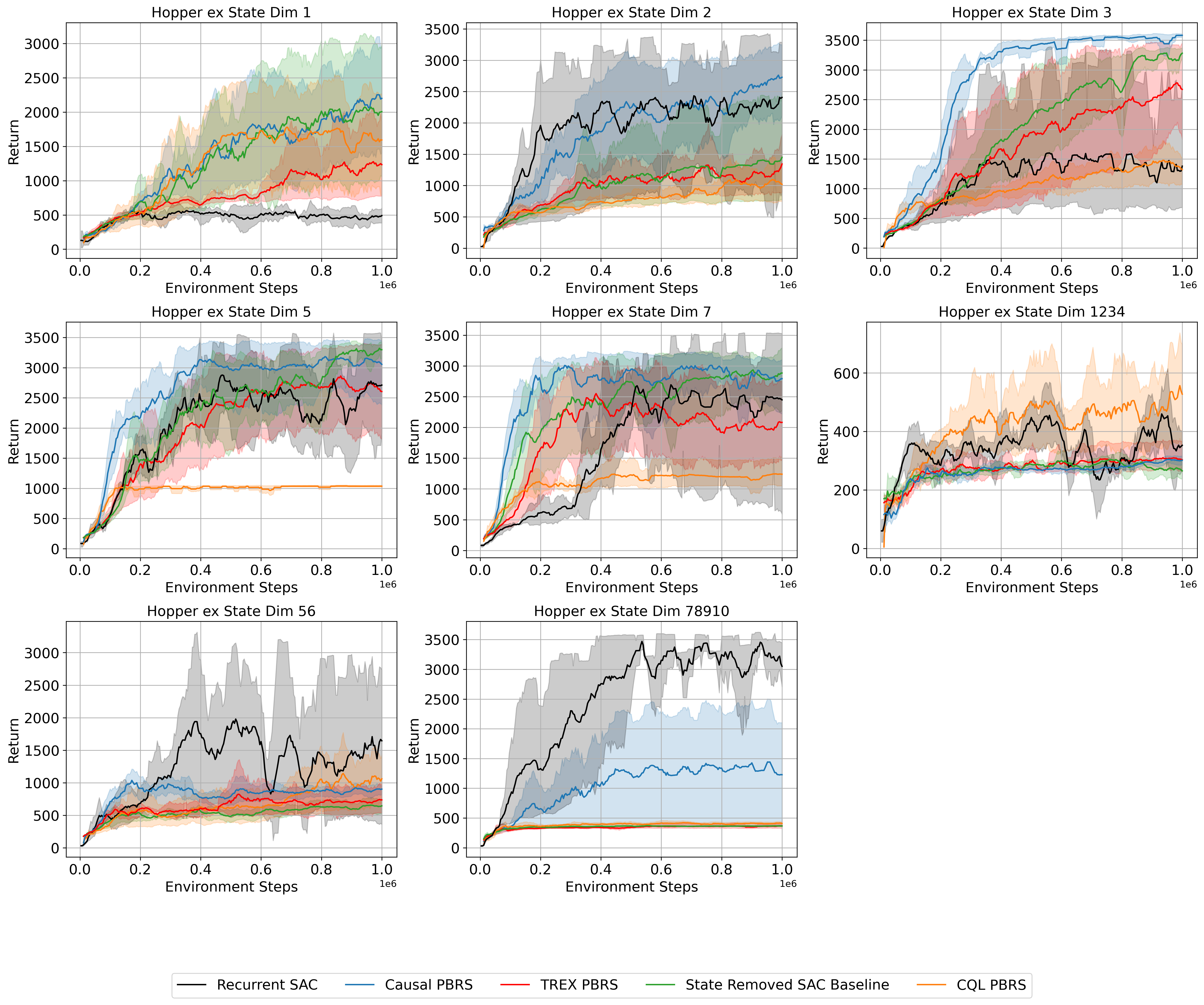} 
    \caption{Confounded Hopper-v5 Experiment Results
    }
    \Description{Confounded Hopper-v5 Experiment Results}
    \label{fig:hopper_graphs}
\end{figure}

\begin{figure}[H]
    \centering
    \includegraphics[width=.9\textwidth]{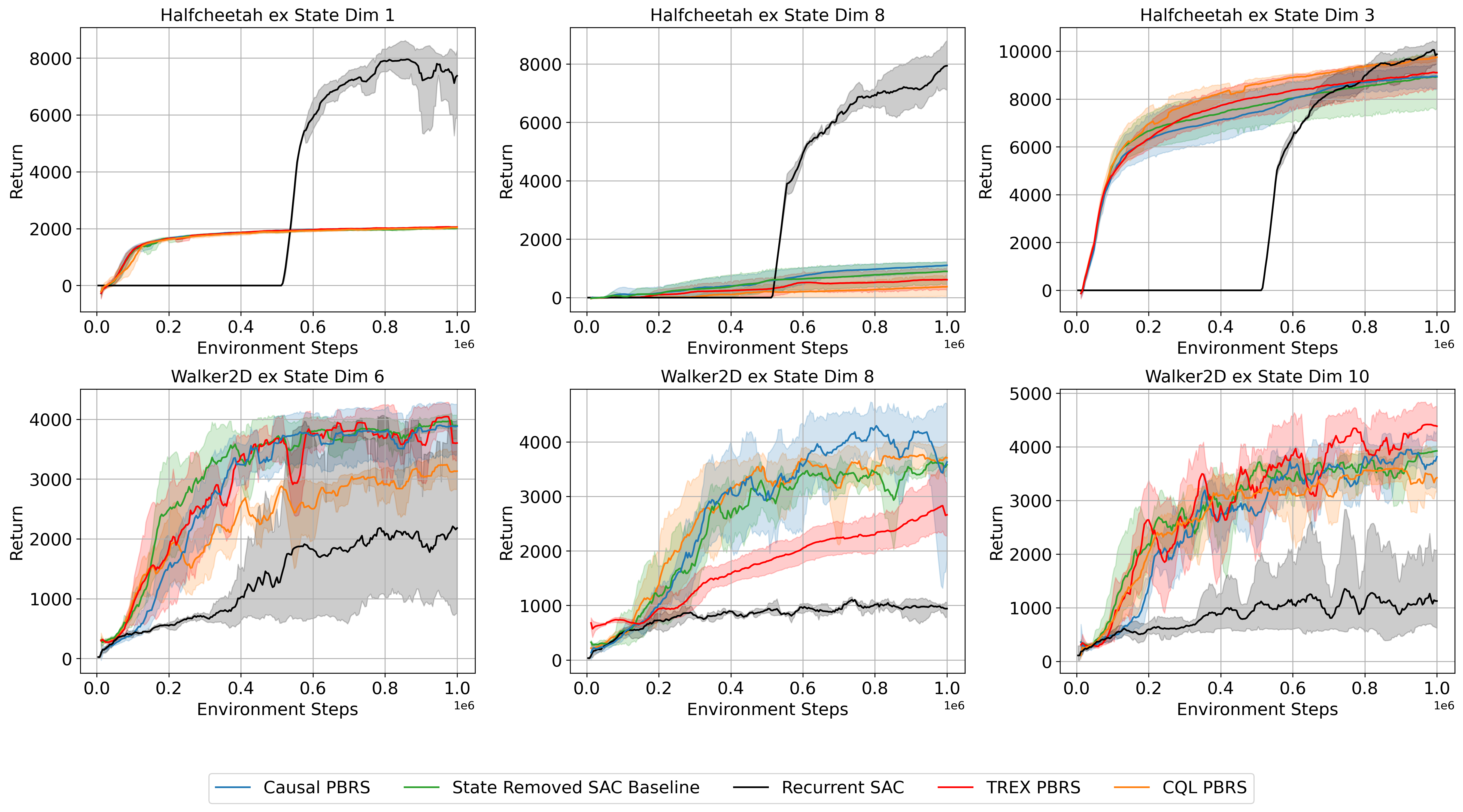} 
    \caption{Confounded HalfCheetah-v5 and Walker2d-v5 Experiment Results
    }
    \Description{Confounded HalfCheetah-v5 and Walker2d-v5 Experiment Results}
    \label{fig:halfcheetah_walker_graphs}
\end{figure}

\begin{figure}[H]
    \centering
    \includegraphics[width=.9\textwidth]{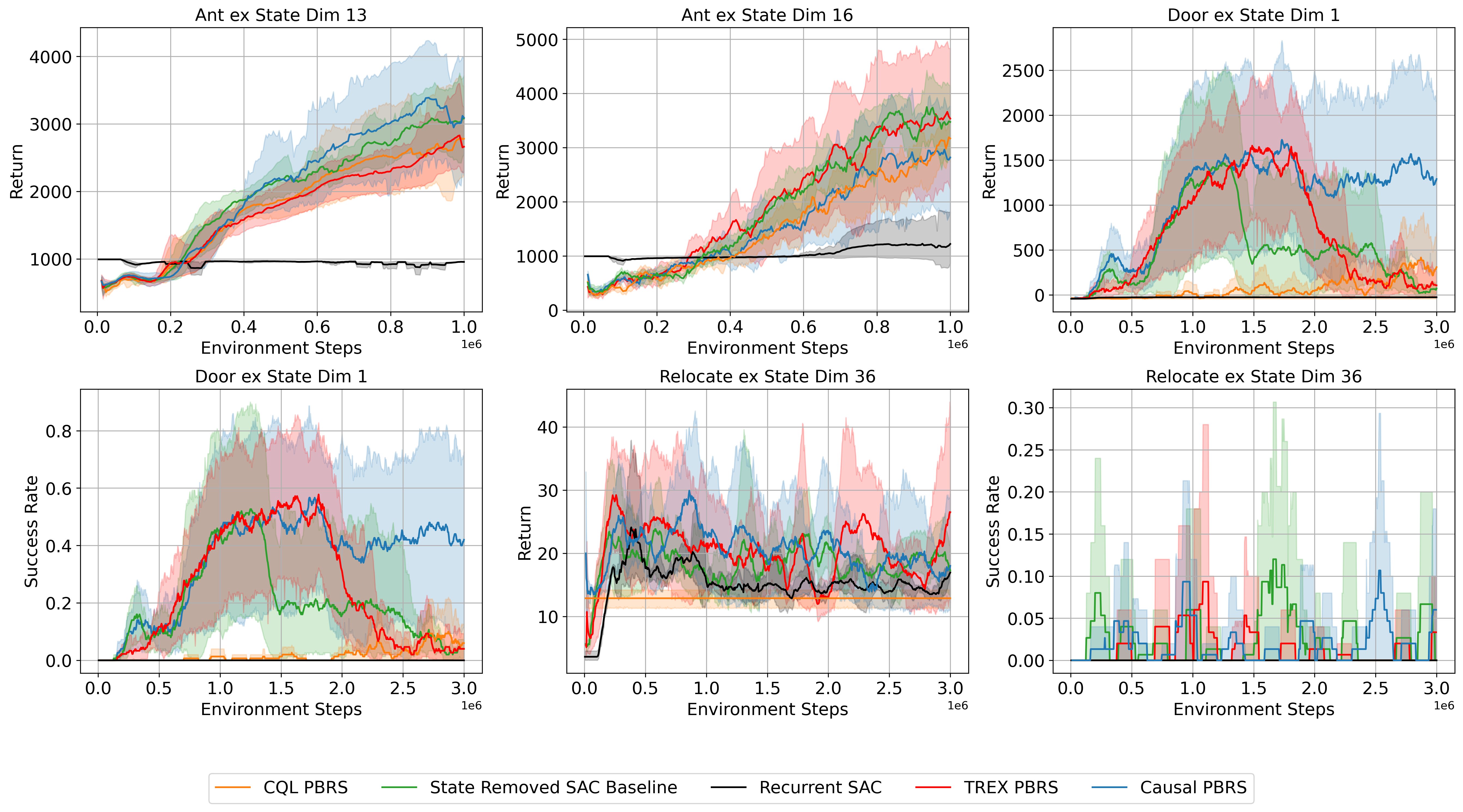} 
    \caption{Confounded Ant-v5, AdroitHandDoor-v1, AdroitHandRelocate-v1 Experiment Results (Returns)
    }
    \Description{Confounded Ant-v5, AdroitHandDoor-v1, AdroitHandRelocate-v1 Experiment Results}
    \label{fig:ant_relocate_graphs}
\end{figure}

\begin{table*}[t]
\centering
\caption{Timesteps of when Baseline and Causal PBRS agent reach best eval}
\label{tab:step_comparison_results}
\adjustbox{max width=\textwidth}{
\begin{tabular}{l|c|c|c}
\toprule
Environment & State Dim Removed & Baseline Steps & Causal Steps \\
\midrule

\multirow{8}{*}{Hopper-v5} 
 & 1 & 944,000 & 988,000 \\
 & 2 & 1,000,000 & 988,000 \\
 & 3 & 948,000 & 1,000,000 \\
 & 5 & 992,000 & 812,000 \\
 & 7 & 904,000 & 764,000 \\
 & 5,6 & 964,000 & 172,000 \\
 & 1,2,3,4 & 780,000 & 960,000 \\
 & 7,8,9,10 & 552,000 & 952,000 \\

\midrule
\multirow{3}{*}{Halfcheetah-v5} 
 & 1 & 1,000,000 & 988,000 \\
 & 3 & 1,000,000 & 996,000 \\
 & 8 & 996,000 & 1,000,000 \\

\midrule
\multirow{3}{*}{Walker2D-v5} 
 & 6 & 984,000 & 952,000 \\
 & 8 & 956,000 & 804,000 \\
 & 10 & 1,000,000 & 928,000 \\

\midrule
\multirow{2}{*}{Ant-v5} 
 & 13 & 912,000 & 904,000 \\
 & 16 & 936,000 & 984,000 \\

\midrule
\multirow{1}{*}{AdroitHandDoor-v1} 
 & 1 & 1,248,000 & 1,732,000 \\

\midrule
\multirow{1}{*}{AdroitHandRelocate-v1} 
 & 36 & 584,000 & 860,000 \\

\bottomrule
\end{tabular}
}
\end{table*}

\begin{table*}[t]
\centering
\caption{Evaluation performance of agents ($\pm$ 1 standard deviation) at final timestep on the 18 confounded environments. Bold Numbers indicate best-performing methods.}
\label{tab:state_removal_results_final_eval}
\adjustbox{max width=\textwidth}{
\begin{tabular}{l|c|c|c|c|c|c|c}
\toprule
Environment & Full State SAC & State Dim Removed & Baseline & CQL & Trex & Recurrent SAC & Causal PBRS (ours) \\
\midrule

\multirow{8}{*}{Hopper-v5} & \multirow{8}{*}{3500} 
 & 1 & 2005 $\pm$ 1253 & 1592 $\pm$ 934 & 1233 $\pm$ 718 & 490 $\pm$ 108 & \textbf{2199 $\pm$ 1022} \\
 &  & 2 & 1450 $\pm$ 1123 & 1018 $\pm$ 365 & 1344 $\pm$ 628 & 2407 $\pm$ 1325 & \textbf{2729 $\pm$ 745} \\
 &  & 3 & 3281 $\pm$ 122 & 1321 $\pm$ 297 & 2671 $\pm$ 1067 & 1379 $\pm$ 1196 & \textbf{3582 $\pm$ 37} \\
 &  & 5 & \textbf{3298 $\pm$ 183} & 1035 $\pm$ 19 & 2605 $\pm$ 1062 & 2711 $\pm$ 743 & 3055 $\pm$ 735 \\
 &  & 7 & \textbf{2885 $\pm$ 691} & 1239 $\pm$ 300 & 2081 $\pm$ 818 & 2448 $\pm$ 1616 & 2793 $\pm$ 638 \\
 &  & 5,6 & 645 $\pm$ 156 & 1060 $\pm$ 481 & 737 $\pm$ 297 & \textbf{1645 $\pm$ 1200} & 903 $\pm$ 84 \\
 &  & 1,2,3,4 & 265 $\pm$ 36 & \textbf{528 $\pm$ 224} & 303 $\pm$ 71 & 353 $\pm$ 63 & 301 $\pm$ 20 \\
 &  & 7,8,9,10 & 367 $\pm$ 11 & 409 $\pm$ 36 & 367 $\pm$ 62  & \textbf{3050 $\pm$ 375} & 1228 $\pm$ 1154 \\

\midrule
\multirow{3}{*}{Halfcheetah-v5} & \multirow{3}{*}{12400} 
 & 1 & 2013 $\pm$ 32 & 2039 $\pm$ 31 & 2057 $\pm$ 25 & \textbf{7378 $\pm$ 1303} & 2051 $\pm$ 22 \\
 &  & 3 & 8931 $\pm$ 1489 & 9748 $\pm$ 177 & 9106 $\pm$ 870 & \textbf{9875 $\pm$ 509} & 8949 $\pm$ 663 \\
 &  & 8 & 895 $\pm$ 495 & 373 $\pm$ 454 & 612 $\pm$ 444 & \textbf{7939 $\pm$ 851} & 1103 $\pm$ 216 \\

\midrule
\multirow{3}{*}{Walker2D-v5} & \multirow{3}{*}{4050} 
 & 6 & \textbf{3893 $\pm$ 263} & 3132 $\pm$ 403 & 3603 $\pm$ 334 & 2181 $\pm$ 1644 & 3884 $\pm$ 551 \\
 &  & 8 & 3580 $\pm$ 272 & \textbf{3708 $\pm$ 308} & 2663 $\pm$ 663 & 942 $\pm$ 136 & 3632 $\pm$ 1781 \\
 &  & 10 & 3925 $\pm$ 137 & 3424 $\pm$ 391 & \textbf{4385 $\pm$ 427} & 1124 $\pm$ 829 & 3814 $\pm$ 611 \\

\midrule
\multirow{2}{*}{Ant-v5} & \multirow{2}{*}{4000} 
 & 13 & 3082 $\pm$ 758 & 2781 $\pm$ 877 & 2663 $\pm$ 663 & 957 $\pm$ 4 & \textbf{3093 $\pm$ 1122} \\
 &  & 16 & \textbf{3474 $\pm$ 936} & 3169 $\pm$ 487 & 3536 $\pm$ 1708 & 1221 $\pm$ 518 & 2818 $\pm$ 1217 \\

\midrule
\multirow{1}{*}{AdroitHandDoor-v1} & \multirow{1}{*}{NA} 
 & 1 & 71 $\pm$ 102 & 308 $\pm$ 485 & 105 $\pm$ 178 & -27 $\pm$ 1 & \textbf{1289 $\pm$ 1252} \\

\midrule
\multirow{1}{*}{AdroitHandRelocate-v1} & \multirow{1}{*}{NA} 
 & 36 & 17 $\pm$ 5 & 13 $\pm$ 2 & \textbf{27 $\pm$ 20} & 17 $\pm$ 4 & 18 $\pm$ 13 \\

\bottomrule
\end{tabular}
}
\end{table*}
\section{Full Experiment Commentary}
\label{sec:full_commentary}
\subsection{Hopper}

We removed the following dimensions:

\begin{enumerate}[label=\textbullet, leftmargin=20pt, topsep=0pt, parsep=0pt, itemsep=1pt]
    \item State Dim 1: The Hopper's torso angle. One of the termination conditions in Hopper is if the torso angle is bounded between $[-0.2, 0.2]$. Without this dimension, the agent is unable to know how to adjust its torso angle to maintain healthy body positions. The Causal PBRS can improve on the baseline by around $20\%$, whereas CQL PBRS / TREX PBRS / Recurrent SAC underperform by $15\% / 38\% / 73\%$. 
    \item State Dim 2: Angle of the thigh joint. Removing this dimension has a large impact on the Hopper's performance, reducing its baseline return by $50\%$ vs the Hopper's performance with full obs capacity using SAC. The Causal PBRS method can drive a large improvement of almost $100\%$ vs the baseline. The Recurrent SAC method performs second best.
    \item State Dim 3: Angle of the leg joint. The leg joint is comparable to the knee, so without it, the Hopper has difficulties learning how to use its lower joints to propel itself forward. Interestingly, even with this dimension removed, the Hopper is able to get high returns, though it takes almost 900,000 steps. The Causal PBRS method is able to recover full state observation SAC's performance, and does so in only around 500,000 steps. All other methods underperform the baseline.
    \item State Dim 5: The Hopper's forward velocity. The Hopper's reward function is heavily dependent on its forward velocity. Therefore, not observing the forward velocity might decrease the critic's ability to estimate the Q-value function. Surprisingly, without this state dimension, the baseline agent can reach a return comparable to a Hopper agent with full observation capabilities. We note that the Causal PBRS agent surpasses a score of 3,000 after only 400,000 time steps, whereas the baseline agent requires almost 900,000 steps. This highlights the Causal PBRS's ability to accelerate training. 
    \item State Dim 7: The angular acceleration of the Hopper's torso. The Causal PBRS's average return is slightly higher than the baseline agent; however, we note the Causal PBRS reaches its performance peak at around 300,000 steps, whereas the baseline agent reaches its peak at 900,000 steps.
    \item State Dim 1,2,3,4: All of the angular positions of the Hopper's joints and Torso. Without knowing the position of its joints, the Hopper agent is unable to learn a meaningful policy. The Causal PBRS method is similarly unable to learn a meaningful policy, suggesting a limit to the Causal PBRS's ability to improve performance in highly confounded environments. Interestingly, the CQL method can generate a slightly better policy, although we note the overall still low return. 
    \item State Dim 5,6: The forward and vertical velocity of the hopper. The Causal PBRS agent can learn a better policy than the baseline, achieving almost double the baseline's return. The CQL method can learn a slightly better policy. The recurrent SAC method is able to learn the best policy, which is likely a result of the lower dimensionality of the environment, which reduces the computational complexity required for backpropagation through time.
    \item State Dim 7,8,9,10: The angular acceleration of the hopper's torso and joints. Interestingly, the baseline Hopper's performance is comparable to its performance when removing dimensions 1,2,3,4, however, the Causal PBRS agent can learn a more meaningful policy (return between 1200-1400). The best performance is the recurrent SAC method, which is able to achieve a very high score, although we know that this is likely because of the reduced dimension (trained on 7 dim instead of full 11) and computational complexity.
\end{enumerate}

\subsection{HalfCheetah}

We removed the following dimensions:
\begin{enumerate}[label=\textbullet, leftmargin=20pt, topsep=0pt, parsep=0pt, itemsep=1pt]
    \item State Dim 1: Angle of the front tip. Without the angle of the front tip (used to know the orientation of the Cheetah), the Half Cheetah agent's return drastically decreases from 12,000 to around 2,000. The baseline, the Causal PBRS methods, and the CQL PBRS methods all perform comparably. The recurrent SAC method greatly outperforms all of the other methods. 
    \item State Dim 3: Angle of the back shin. The back shin in the middle joint on the HalfCheetah's back leg. Removing this dimension from the observation dimension decreases performance slightly ($10\%$). Similar to removing state dimension 1, the Causal PBRS method performs similarly to the baseline; however, the recurrent SAC method can learn the best policy.
    \item State Dim 8: Velocity of the x-coordinate of the front tip. Unlike the Hopper, removing the forward velocity greatly reduces the HalfCheetah's return. The Causal PBRS method achieves a higher return vs the baseline policy. The recurrent SAC method greatly outperforms all of the other methods.
\end{enumerate}

\subsection{Walker2D}

We removed the following dimensions:
\begin{enumerate}[label=\textbullet, leftmargin=20pt, topsep=0pt, parsep=0pt, itemsep=1pt]
    \item State Dim 6: Angle of the left leg joint. Removing the left leg joint has a limited effect on the Walker2d. Both the Causal PBRS and Baseline methods perform comparably, but the CQL method underperforms both by almost $20\%$. The TREX PBRS method is able to achieve the best performance, although this is expected, given the lower amount of confounding bias (given the baseline and full-state SAC performance are comparable). Therefore, the TREX rewards are less biased, and the PBRS function does not inject noise into the environment reward. Despite the lower levels of confounding bias, the recurrent SAC method is unable to learn a useful policy highlighting the method's inability to scale to higher dimensions and more complicated environments (such as where the episode can be terminated, unlike the half-cheetah environment).
    \item State Dim 8: Velocity of the x-coordinate of the torso. Removing the forward velocity dimension of the Walker2d slightly decreases performance. The Causal PBRS method achieves a $20\%$ improvement over the baseline method. The CQL method improves on the baseline by 3\%, whereas the other two methods under perform the baseline.
    \item State Dim 10: Angular velocity of the angle of the torso. Removing the left leg joint observation dimensions has a limited effect on the Walker-2d.  The Causal PBRS, Online Baseline, and CQL PBRS perform comparably. The TrexPBR method again achieves the highest level of performance, similar to when removing dimension 6, while the recurrent SAC method greatly under performs the baseline.
\end{enumerate}

\subsection{Ant}

We removed the following dimensions:
\begin{enumerate}[label=\textbullet, leftmargin=20pt, topsep=0pt, parsep=0pt, itemsep=1pt]
    \item State Dim 13: Velocity of the x-coordinate of the torso. Similar to other environments, the Causal PBRS method helps the Ant recover some of the performance loss from the loss of its forward velocity observation, unlike the other methods, which all under perform, in particular the recurrent SAC methodology.
    \item State Dim 16: x-coordinate angular velocity of the torso. The Causal PBRS underperforms the baseline; however, the baseline performs much better than the baseline when removing dimension 13, suggesting less confounding bias in this setting and hence a potential reason for the poor performance. The CQL PBRS method slightly outperforms the Causal PBRS method, but still under performs the baseline.

\end{enumerate}

\subsection{Adroit Door}

We removed the following dimensions:
\begin{enumerate}[label=\textbullet, leftmargin=20pt, topsep=0pt, parsep=0pt, itemsep=1pt]
    \item State Dim 1: Angular position of the horizontal arm joint. The baseline agent and Causal PBRS agent perform comparably. The Causal PBRS method outperforms the baseline by $17\%$, while the CQL / TREX PBRS methods lags behind. Additionally, the causal PBRS method is the only method that is able to avoid the catastrophic forgetting and is able to maintain a similar performance throughout all 3m steps, whereas the other methods are unable to do so. The current SAC method is unable to learn any meaningful policy.
\end{enumerate}

\subsection{Adroit Relocate}
We removed the following dimensions:
\begin{enumerate}[label=\textbullet, leftmargin=20pt, topsep=0pt, parsep=0pt, itemsep=1pt]
    \item State Dim 36: x positional difference from the ball to the target. While the agent is still able to infer this position through other observations (namely dimensions 30 and 33), the overall performance is still low. The Causal PBRS method outperforms all methods, although we note overall performance is still very low.

\end{enumerate}

\end{document}